\documentclass[]{article}
\usepackage[margin=0.8in]{geometry}

\title{Multi-Scale Zero-Order Optimization of Smooth Functions in an RKHS}

\author{Shubhanshu Shekhar   \\  \ttt{shshekha@eng.ucsd.edu}
\and 
Tara Javidi  \\  \ttt{tjavidi@eng.ucsd.edu}
} 
   
\date{}

\usepackage{xcolor} 
\usepackage{wrapfig}

\usepackage{natbib}
\usepackage{soul}
\setstcolor{red}

\usepackage{amsmath}
\usepackage{amsthm}
\usepackage{amssymb}
\usepackage{bm} 
\usepackage{bbm} 
\usepackage{mathtools}
\usepackage{enumitem}
\usepackage{longtable}
\usepackage[boxruled, vlined,linesnumbered]{algorithm2e}
\usepackage{hyperref}

\usepackage{subcaption}
\usepackage{diagbox}
\usepackage{makecell}
\makeatletter
\g@addto@macro\normalsize{%
\setlength\abovedisplayskip{6pt} 
\setlength\belowdisplayskip{6pt}
\setlength\abovedisplayshortskip{6pt} 
\setlength\belowdisplayshortskip{6pt}
}
\makeatother

\newtheorem{theorem}{Theorem}
\newtheorem{lemma}{Lemma}
\newtheorem{proposition}{Proposition}

\newtheorem{corollary}{Corollary}
\newtheorem{assumption}{Assumption}

\theoremstyle{definition}
\newtheorem{definition}{Definition}
\newtheorem{remark}{Remark}

\newcommand{\lp}{\left (} 
\newcommand{\rp}{\right )} 

\newcommand{\mc}[1]{\mathcal{#1}}
\newcommand{\mbb}[1]{\mathbb{#1}}

\newcommand{\reals}{\mathbb{R}}

\newcommand{\X}{\mathcal{X}}

\DeclareMathOperator*{\argmax}{arg\,max}
\DeclareMathOperator*{\argmin}{arg\,min}

\newcommand{\tbf}[1]{\textbf{#1}}

\makeatletter
\newcommand{\pushright}[1]{\ifmeasuring@#1\else\omit\hfill$\displaystyle#1$\fi\ignorespaces}
\newcommand{\pushleft}[1]{\ifmeasuring@#1\else\omit$\displaystyle#1$\hfill\fi\ignorespaces}
\makeatother

\newenvironment{proofoutline}
 {\proof[Proof outline]}
 {\endproof}

\newcommand*{\lpgpucb}{\texttt{LP-GP-UCB}\xspace}
\newcommand*{\heuristic}{\texttt{Heuristic}\xspace}
\newcommand{\f}{f}
\newcommand{\xstar}{x^*}

\newcommand{\partition}{\texttt{Partition}\xspace}
\newcommand{\expand}{\texttt{ExpandAndBound}\xspace}
\newcommand{\recommend}{\texttt{Recommend}\xspace}
\newcommand{\localpoly}{\texttt{LocalPoly}\xspace}

\newcommand{\maxerr}{\texttt{MaxErr}\xspace}

\newcommand{\ttt}[1]{\texttt{#1}}

\newcommand{\utone}{u_{t,E}^{(1)}}
\newcommand{\uttwo}{u_{t,E}^{(2)}}

\newcommand{\DYE}{\mc{D}_{\mc{Y}}^{(E)}} 
\newcommand{\DXE}{\mc{D}_{\domain}^{(E)}}

\newcommand{\domain}{\mathcal{X}} 
\newcommand{\noise}{\eta}
\newcommand{\algo}[1][]{\mc{A}_{#1}}
\newcommand{\recc}[1][n]{z_{#1}}

\newcommand{\creg}[1][n]{\mathcal{R}_{#1}} 
\newcommand{\sreg}[1][n]{\mathcal{S}_{#1}} 
\newcommand{\ireg}[1][t]{\texttt{reg}_{#1}}

\newcommand{\rkhs}[1][K]{\mc{H}_{#1}}

\newcommand{\holder}{H\"older } 

\newcommand{\holdersp}[2][k]{\mc{C}^{#1, #2}}

\newcommand{\Tau}{\mathrm{T}}

\newcommand{\kse}{K_{\text{SE}}}
\newcommand{\kmat}[1][\nu]{K_{#1}}

\newcommand{\matern}{Mat\'ern }
\newcommand{\krq}{K_{\text{RQ}}} 

\newcommand{\kgexp}{K_{\gamma\text{-exp}}} 
\newcommand{\kpp}{K_{\text{pp},q}}
\newcommand{\sobolev}{W^{\nu+D/2, \nu}}
\newcommand{\holderspk}{\mc{C}^{k,\alpha}}
\newcommand{\igain}[1][n]{\gamma_{#1}}

\newcommand{\poly}{\mc{P}_D^k}
\newcommand{\Oh}[1]{\mathcal{O}\lp #1 \rp}
\newcommand{\tOh}[1]{\tilde{\mathcal{O}} \lp #1 \rp }

\newcommand{\propref}[1]{Proposition~\ref{#1}}
\newcommand{\algoref}[1]{Algorithm~\ref{#1}}
\newcommand{\thmref}[1]{Theorem~\ref{#1}}

\begin{document}

\maketitle
\begin{abstract}
We aim to optimize a black-box function $f:\domain \mapsto \reals$ under the assumption that $f$ is \holder smooth and has bounded norm in the Reproducing Kernel Hilbert Space (RKHS) associated with a given kernel $K$. This problem is known to have an agnostic Gaussian Process (GP) bandit interpretation in which an appropriately constructed GP surrogate model with kernel $K$ is used to obtain an upper confidence bound (UCB) algorithm. 
In this paper, we propose a new algorithm (\lpgpucb) where the usual GP surrogate model is augmented with Local Polynomial (LP) estimators of the \holder smooth function $f$ to construct a multi-scale upper confidence bound guiding the search for the optimizer. We analyze this algorithm and derive high probability bounds on its simple and cumulative regret. 
We then prove that the elements of many common reproducing kernel Hilbert spaces are \holder smooth and obtain the corresponding \holder smoothness parameters, and hence, specialize our regret bounds for several  commonly used and practically relevant kernels. When specialized to the Squared Exponential (SE) kernel, \lpgpucb matches the optimal performance, while for the case of \matern kernels~$(\kmat)_{\nu>0}$, it results in uniformly tighter regret bounds for all values of the smoothness parameter $\nu>0$. Most notably, for certain ranges of $\nu$, the algorithm achieves near-optimal bounds on simple and cumulative regrets, matching the algorithm-independent lower bounds up to poly-logarithmic factors, and thus closing the large gap between the existing upper and lower bounds for these values of $\nu$. Additionally, our analysis provides the first explicit regret bounds, in terms of the budget $n$, for the Rational-Quadratic~(RQ) and Gamma-Exponential~(GE). Finally, experiments with synthetic functions as well as a Convolutional Neural Network
hyperparameter tuning task demonstrate the practical benefits of our multi-scale partitioning approach over some existing algorithms numerically. 

\end{abstract}

\section{Introduction}
\label{sec:introduction}
Consider the problem  of 
maximizing a black-box
objective function $\f: \domain \mapsto \reals$ which can only be accessed through a \emph{noisy} \emph{zero-order oracle}
which upon querying the objective function $\f$ at an arbitrary point $x \in \domain$ provides an  observation $y_x = \f(x) + \noise_x$.
Our goal is to design a query-point selection strategy $\algo$, which can efficiently learn about a maximizer $\xstar$  of $\f$ given a finite query (or evaluation) budget of $n$.
After the evaluation budget is exhausted, the algorithm $\algo$ must recommend a point, denoted by $z_n$. 
Two commonly used measures for the performance of a sampling strategy are its \emph{simple regret} $\sreg$ and its \emph{cumulative regret} $\creg$ defined as $\sreg=\f(\xstar) - f(z_n)$ and $\creg = \sum_{t=1}^n f(\xstar) - f(x_t)$. 
Both $\sreg$ and $\creg$ are  random variables, and  we will derive bounds on these quantities which  hold with probability at least $1-\delta$, for some confidence parameter $\delta \in (0,1)$. 

This problem is intractable without any regularity assumptions on the objective function $f$; in this paper, we assume that $f$ has a bounded norm  in the RKHS associated with a kernel $K$ (denoted by $\rkhs$). This formulation is referred to as the \emph{agnostic} Gaussian Process (GP) bandit problem in literature~\citep{srinivas2012information} where a GP surrogate can be used to derive an upper confidence bound (UCB) on the function which then is utilized to guide the search for the optimizer. 
In this paper, we propose a new algorithm  which exploits  smoothness properties of functions lying in the $\rkhs$ of commonly used kernels, to obtain tighter bounds on both $\sreg$ and $\creg$. 
 Our work is  motivated by  a large gap, in the agnostic GP bandits setting, between the best known upper bounds  and the algorithm-independent lower bounds on the regret for the \matern~family, the most commonly used family of kernels~\citep{scarlett2017lower}. 
Next, we present the notations used in this paper, and then conclude the section with an overview of our contributions in Sec.~\ref{subsec:overview} and discussion of related work in Sec.~\ref{subsec:prior}. 

\tbf{Notations.} Recall that $f$ is the objective function mapping $\mc{X} = [0,1]^D$ to $\mc{Y} = \mbb{R}$. The function $f$ can be accessed through a noisy evaluations of the form $y = f(x) + \eta$, where the additive noise $\eta$ is assumed to be $\sigma$ sub-gaussian. We will use the term \emph{cell} to refer to subsets $E$ of $\domain$ of the form $E = \{x \in \domain \;:\; \|x - x_E\|_\infty \leq r_E/2\}$. The terms $x_E$ and $r_E$ shall be referred to as the \emph{center} and the \emph{side-length} of $E$. 
Given a kernel  $K$, we shall use the term $\rkhs$ and $\|\cdot\|_K$ to denote the associated RKHS and the RKHS norm respectively~(see App.~\ref{appendix:preliminaries} for definitions).  For $k \in \mbb{N}$ and $0<\alpha\leq 1$, we use $\mc{C}^{k,\alpha}$ to denote the \holder space with parameters $k$ and $\alpha$ (see App.~\ref{appendix:preliminaries} for definition). 
Let $\mc{D} = \{ (x_i, y_i) \; : \; 1 \leq i \leq m\} \subset \mc{X} \times \mc{Y}$ denote a labelled data set, and introduce $\mc{D}_\mc{X}\coloneqq \{x \; : \exists y \in \mc{Y}, (x, y) \in \mc{D}\}$ and $\mc{D}_\mc{Y} \coloneqq \{ y \;:\; \exists x \in \mc{X}, (x,y) \in \mc{D}\}$. For a cell $E \subset \mc{X}$, we use $\mc{D}^{(E)}$ to denote the subset of $\mc{D}$ with $x_i \in E$. The sets $\mc{D}_{\X}^{(E)}$ and $\mc{D}_{\mc{Y}}^{(E)}$ are also defined in an analogous manner. For  positive integers $k$ and $D$, we use $\mc{P}_D^k$ to denote the set of all polynomials in $D$ variables of degree $k$. Given $f:\domain \mapsto \mbb{R}$ and $E \subset \domain$, we define the quantity $\Phi_k(f, E) = \inf_{p \in \mc{P}_D^k}\sup_{x \in E} |f(x) - p(x)|$ as the smallest uniform approximation error of $f$ with $p$ in $\mc{P}_D^k$.



\subsection{Overview of Results}
\label{subsec:overview}
We first formally state the assumptions on objective function $f$ and the observation noise. 
\begin{assumption}
\label{assump:main} 
We make the following  assumptions: 
\tbf{(A1.1)} $\f \in \rkhs$ for some known kernel $K$ and furthermore $\|\f\|_K \leq B$  for some known constant $B>0$. 
\tbf{(A1.2)} $\f \in \holdersp{\alpha}$
for $k \in \mbb{N}\cup\{0\}$ and $\alpha \in (0,1]$ with $\|f\|_{\holdersp{\alpha}} \leq L$ for some known $L>0$. 
\tbf{(A1.3)} the observation noise $(\eta_t)_{t \geq 0}$ are $i.i.d$ and  $\sigma^2$-sub-Gaussian for some known constant $\sigma^2>0$. 
\end{assumption}
\begin{remark}
\label{remark:assumptions}
Assumption \tbf{(A1.1)} means that $f$ has \emph{low complexity} as measured through the RKHS norm, while \tbf{(A1.3)} requires the observation noise to have \emph{light tails}. These two assumptions are standard, and have been employed in most prior works in the agnostic GP bandits literature. The assumption~\tbf{(A1.2)} requires $f$ to satisfy further \holder smoothness property. In Propositions~\ref{prop:embed1}~and~\ref{prop:embed2}, we show that 
\tbf{(A1.2)} is superfluous 
for all common and practically relevant kernels as it can be derived as a consequence of~\tbf{(A1.1)}. 
\end{remark} 

\noindent Next, we list the main contributions of this paper:
\begin{itemize}[leftmargin=*]
\item We propose a new algorithm (\lpgpucb) for agnostic GP bandits, which combines the  usual GP surrogate model (as in GP-UCB algorithm of \cite{srinivas2012information}) with Local Polynomial (LP) estimators to guide the search for the optimizer of $\f$.~(Sec.~\ref{sec:LPGPUCB})
\item We analyze \lpgpucb algorithm under Assumption~\ref{assump:main}, and show (Theorem~\ref{theorem:regret1} in Sec.~\ref{subsec:analysis}) that the $\sreg$ and $\creg$ can be bounded by the minimum of two terms: 
(i) one depending on the maximum \emph{inforamtion gain} $\igain$ (Eq.~\ref{eq:igain0}) of kernel $K$ from Assump.~\tbf{A1.1} similar to previous results, and
(ii) second depending on the smoothness parameters from Assumption~\tbf{A1.2}. 
\item As a result of Propositions~\ref{prop:embed1}~and~\ref{prop:embed2}, the analysis 
in Theorem~\ref{theorem:regret1} can be specialized. \lpgpucb algorithm matches the performance of existing algorithms, such as GP-UCB, for the case of $K=\kse$ for which the $\igain$ dependent bound is known to be near-optimal. 
Furthermore, for the practically relevant case of \matern kernels
$(\kmat)_{\nu>0}$, our results reduce the gap between the existing upper and lower bounds on both  $\sreg$ and $\creg$ for all values of $\nu>0$.  In particular,  for $\kmat$ with $\nu\leq D(D+1)/2$ (resp. $\nu\leq 1$) we show that the \lpgpucb algorithm achieves near-optimal bounds on  $\sreg$~(resp. $\creg$) thus closing the large gap to the algorithm-independent lower bounds of \cite{scarlett2017lower}. (Sec.~\ref{subsec:analysis}).
\item  Besides $\kse$ and $(\kmat)_{\nu>0}$, our approach also allows us to derive the first explicit
(in $n$)  bounds  on $\sreg$ and $\creg$ for  other important class of kernels such as the rational quadratic (RQ), gamma exponential~(GE) and piecewise-polynomial (PP) kernels.~(Sec.~\ref{subsec:lp0}).
\end{itemize}


At a high level, \lpgpucb algorithm adaptively constructs and updates a non-uniform partition $(\mc{P}_t)$ of the domain $\domain = [0,1]^D$. Exploiting the smoothness properties of
the objective function $f$, the \lpgpucb algorithm augments the usual GP surrogate with local polynomial estimators over this non-uniform multi-scale partition which allows us to focus sampling in the near-optimal regions of the domain $\domain$.

%
%
In \thmref{theorem:regret1}, we show  that the regret of LP-GP-UCB can be upper bounded by the minimum of two terms. The first is the usual informational bound as a function of \emph{maximum information gain} denoted by $\igain$:
\begin{equation}
\label{eq:igain0}
\igain \coloneqq \max_{S \subset \domain, |S|=n} I \lp y_S; \f \rp
\end{equation}
where $I(y_S;f)$ denotes the mutual information between $f$ (considered as a sample from a GP) and the noisy observation vector $y_S$. The second term in our upper bound analysis depends on the smoothness parameters $(k,\alpha)$. Together, these two bounds   imply that   LP-GP-UCB achieves near-optimal regret bounds for  SE kernel, and reduces the gap between the upper and lower bounds for  \matern~kernels for all $\nu>0$. Furthermore, for certain parameter ranges, our algorithm achieves near optimal performance~(see Section~\ref{subsec:improved} for details).

Besides SE and \matern~kernels, there also exist some other important kernels such as rational-quadratic~{\small(RQ)}, gamma-exponential~{\small (GQ)} and piecewise-polynomial~{\small (PP)} kernels. Further details on these kernels can be found in \citep[Ch.~4]{williams2006gaussian}. For these  kernels, 
prior work and 
analysis 
fails to provide analytic bounds on their regret since the requisite bounds on $\igain$ for these kernels are unknown.
However, our alternative approach, relying on the embedding result of Proposition~\ref{prop:embed2}, allows us to derive the first explicit in $n$ regret bounds for these kernels. The details are in Sec.~\ref{subsec:lp0}. 


\subsection{Related Work}
\label{subsec:prior}
\citep{srinivas2012information} proposed the GP-UCB algorithm for this problem motivated by the UCB strategy for multi-armed bandits (MABs)~\citep{auer2002finite}. They also derived the following generic upper bounds on $\creg$ in terms of  the \emph{maximum information gain} ($\igain$) 
\begin{equation}
\mc{R}_n = \tilde{\mc{O}}\lp \sqrt{n} \igain \rp. 
\label{eq:igain}
\end{equation}
where $\igain$ is defined in~\eqref{eq:igain0}. Following \citep{srinivas2012information}, researchers have analyzed several extensions of the basic GP bandit problem. These extensions study the GP bandit problem with \emph{parallel } observations~\citep{desautels2014parallelizing, contal2013parallel}, contextual information \citep{krause2011contextual}, additive assumption~\citep{kandasamy2015high}, robust-optimization~\citep{bogunovic2018adversarially} and multi-fidelity observations~\citep{kandasamy2016gaussian}. \citep{chowdhury2017kernelized} proposed an improved version of the GP-UCB algorithm by deriving a self-normalized concentration inequality. 
All these results, suitably modify the general analytical approach of \citep{srinivas2012information} to derive $\igain$-dependent regret bounds. An alternative approach was taken in the recent work by \citep{janz2020bandit}, who designed an algorithm which adaptively partitions the input space and fits independent GP models in each element of the partition. This structured approach to sampling yields an improved bound on $\igain$,  results in tighter regret bounds.

For the case of SE kernel, the analysis of \citep{srinivas2012information} provides an upper bound of ${\mc{O}}( \sqrt{ n (\log n)^{2D}})$ on $\creg$, which matches the lower bound of $\Omega ( \sqrt{ n (\log n)^{D/2} })$ derived by \cite{scarlett2017lower} up to poly-logarithmic factors. However, for the \matern~family of kernels, \citep{srinivas2012information}  obtained an upper bound $\creg  = \tilde{\mc{O}}(n^{a_\nu})$ for $\nu>1$, where $a_\nu = \frac{1}{2} + \frac{D(D+1)}{2\nu + D(D+1)}$ which is always larger than the lower bound of $\Omega(n^{b_\nu})$ with $b_\nu = \frac{\nu+D}{2\nu+D}$ derived by \cite{scarlett2017lower}. Moreover, for $\nu, D$ such that $D(D+1) \geq 2\nu$ the upper bound on $\creg$ is not even sublinear in $n$. This issue was fixed by the adaptive partitioning approach of \citep{janz2020bandit}, who obtained a bound of $\creg = \tOh{n^{e_\nu}}$ for $\nu>1$, where $e_\nu = \frac{D(2D+3) + 2\nu}{D(2D+4) + 4\nu}$. Unlike the bounds of \citep{srinivas2012information, chowdhury2017kernelized}, this is sublinear for  $\nu>1$ and $D\geq 1$. Finally, we note the above-mentioned bounds on $\creg$ also imply corresponding bounds on $\sreg$ of the order $(1/n)\creg$ by employing a point-recommendation rule which returns the evaluated point with minimum posterior variance.


\section{\lpgpucb Algorithm}
\label{sec:LPGPUCB}
We now describe the steps of our proposed algorithm, \lpgpucb (pseudo-code in \algoref{algo:algo1}). Recall that the algorithm assumes that the unknown objective function $f \in \rkhs$ with $\|f\|_K \leq B$  and $f \in \mc{C}^{k, \alpha}$ with $\|f\|_{\mc{C}^{k,\alpha}} \leq L$ for some $L, B>0$. More specifically, the algorithm requires the following inputs.

\begin{wrapfigure}{L}{0.45\textwidth}
\begin{minipage}{0.45\textwidth}
\begin{small}
\begin{algorithm}[H]
\SetAlgoLined
\SetKwInOut{Input}{Input}\SetKwInOut{Output}{Output}
\SetKwFunction{LocalPoly}{LocalPoly}
\SetKwFunction{Expand}{Expand}
\newcommand\mycommfont[1]{\ttfamily\textcolor{blue}{#1}}
\SetCommentSty{mycommfont}

\KwIn{$n$, $K$, $B$, $(k,\alpha)$, $L$,  $\rho_0$.}
 \tbf{Initialize:}~ $t=1$, $n_e=0$, $\mc{P}_t = \{ \domain \}$, $u^{(0)}_{\domain}=+\infty$ $\mc{D}_t = \emptyset$\;
 \While{$n_e < n$}{
\For{$E \in \mc{P}_t$}
{  
Draw $x_{t,E} \sim \texttt{Unif}(E)$ \\
$U_{t,E} = \min\{u^{(0)}_E,\; u_{t,E}^{(1)},\; u_{t,E}^{(2)}\}$\\
} 
$E_t \in \argmax_{E \in \mc{P}_t}\; U_{t,E}$\;  \label{algoline:candidate} 
$x_t = x_{t,E_t}$
\BlankLine
$\mc{P}_1  = \mc{P}_2 =  \{E_t\}$ \\
 \uIf{$\beta_n \sigma_t(x_t) < L (\sqrt{D}r_E)^{\alpha_1}$ \texttt{AND} $r_{E_t} \geq \rho_0$}{\label{algline:cond1}
 $\ttt{val} = u_{t,E_t}^{(1)}, \quad$ $\ttt{flag}=1$\\
 $ \mc{P}_2 = \expand(E_t, \texttt{flag}, \ttt{val}, \mc{D}_t, \ldots)$ \label{algline:flag1} \\
  }
  \BlankLine
  %
  \uElseIf{$b_t(E_t) \leq L(\sqrt{D}r_E)^{\alpha_1}$ \texttt{AND} $r_{E_t} \geq \rho_0$} { \label{algline:cond2}
  $\ttt{val} = u_{t,E_t}^{(2)}, \quad$ $\ttt{flag}=1$ \\
 $\mc{P}_2 = \expand(E_t, \texttt{flag}, \ttt{val}, \mc{D}_t, \ldots)$ \label{algline:flag2}\\
 \BlankLine
  }
    \uElseIf{$b_t(E_t) \leq L(\sqrt{D}r_E)^{k+\alpha}$ \texttt{AND} $ r_{E_t} \in [\frac{1}{n}, \rho_0)$} { \label{algline:cond3}
  $\ttt{val} = +\infty, \quad$ $\ttt{flag}=2$ \\
 $\mc{P}_2 = \expand(E_t, \texttt{flag}, \ttt{val}, \mc{D}_t, \ldots)$ \label{algline:flag3}\\
 \BlankLine
  }
  \Else{
 {Observe} $y_t = f(x_t) + \eta_t$\; 
{Update} $\mu_t, \sigma_t$\; 
$n_e \leftarrow n_e + 1, \quad$  $\mc{D}_t \leftarrow \mc{D}_t \cup \{(x_t, y_t)\}$ \\
  }
$\mc{P}_t \leftarrow \lp \mc{P}_t \setminus \mc{P}_1 \rp \cup \mc{P}_2, \quad$  $t \leftarrow t+1$ \\
 }
\KwOut{ $z_n$ using \recommend~function (Def.\ref{def:recommend}).}
 \caption{\lpgpucb Algorithm}
 \label{algo:algo1}
\end{algorithm}
\end{small}
\end{minipage}
\end{wrapfigure}



\noindent\tbf{Inputs.} \algoref{algo:algo1} takes in as inputs the query (or evaluation) budget $n$, the kernel $K$, the parameter $B$ which is a  bound on $\|f\|_{K}$, the noise parameter $\sigma$,  an integer $k$ and an $\alpha \in (0,1]$ for the local polynomial estimator, the parameter $L$ which is an upper bound on $\|f\|_{\holdersp{\alpha}}$ and a real-number $\rho_0 \geq (\igain/\sqrt{LnD^{\alpha_1}})^{1/\alpha_1}$ for $\alpha_1 \coloneqq \max \{\alpha, \min\{1, k\}\}$. 

\begin{remark}
\label{remark:rho0}
 For kernels such as $\kse$ and $\kmat$, explicit upper bounds on $\igain$   derived in \citep[Theorem~5]{srinivas2012information} can be used  directly for the input term $\rho_0$ to \algoref{algo:algo1}. For kernels such as RQ and GE, for which we don't have explicit bounds on $\igain$, we can still  get a numerical bound on $\igain$ to implement \algoref{algo:algo1} by using the fact that $\igain \leq (1-1/e)^{-1} I(y^{G}_{[1:n]}; f)$ where $y_{[1:n]}^{G}$ denotes the $n$ observations according to the \emph{greedy} selection rule, i.e., the strategy which at time $t$ selects the point $x_t \in \argmax_{x \in\domain} \sigma_t(x)$.  Thus before initializing Algorithm~\ref{algo:algo1}, we can run $n$ steps of the \emph{greedy} rule to compute the required upper bound on $\igain$. 
\end{remark}

\tbf{Steps of the Algorithm.} The \lpgpucb algorithm proceeds in the following steps: 

   \;$\bullet$\; 
    The algorithm maintains a partition, $\mc{P}_t$ of the domain $\domain$ at any time $t$, and to each cell $E$ in the partition, it assigns a term $u_E^{(0)}$ which is an upper bound on the maximum $f$ value in the cell, calculated  using local estimates based on  prior observations. At $t=1$, $\mc{P}_t$ is initialized as $\{\domain\}$ and $u_{\domain}^{(0)}$ is set to $+\infty$. As new cells are added to $\mc{P}_t$, the values of $u_E^{(0)}$ are decided by the \expand~algorithm~(pseudo-code in \algoref{algo:expand}). 

   \;$\bullet$\; 
    For every $t \geq 1$, the algorithm loops through all the cells in $\mc{P}_t$, and constructs a UCB denoted by $U_{t,E}$, by taking the minimum of three terms: $u_E^{(0)}$, $u_{t,E}^{(1)}$ and $u_{t,E}^{(2)}$. The terms $\utone$ and $\uttwo$ are defined as  
    \begin{align*}
\utone &= \mu_t(x_{t,E}) + \beta_n \sigma_t(x_{t,E}) + L (\sqrt{D}r_E)^{\alpha_1}, \\ 
         \uttwo &= \hat{\mu}_t(E) + b_t(E) + L (\sqrt{D}r_E)^{\alpha_1}. \label{eq:uttwo}
    \end{align*}
    In the above display, $x_{t,E}$ is a point drawn uniformly from $E$ and $\mu_t$ and $\sigma_t$ are the posterior mean and variance of the surrogate GP model.  The term $\hat{\mu}_t(E)$ is the empirical estimate of the average $f$ value in the cell, i.e., $\hat{\mu}_t(E) = \frac{1}{|\DYE|} \sum_{y \in \mc{D}^{(E)}_{\mc{Y}}} y$, and $b_t(E) = \sqrt{ 2 \log(n/\delta)/n_E}$ is the length of the confidence interval of the average $\f$ value in $E$.  
   
   \;$\bullet$\; 
    Next, the algorithm selects a candidate cell $E_t$ and the corresponding point $x_t$ with the largest value of $U_{t,E}$. 
    
   
   \;$\bullet$\; 
    Having chosen the cell $E_t$, it decides whether we need to update the partition $\mc{P}_t$ by locally expanding the cell $E_t$ (Lines~\ref{algline:cond1},~\ref{algline:cond2}~and~\ref{algline:cond3} of Alg.~\ref{algo:algo1}). The pseudo-code of the \expand~algorithm  in  Algorithm~\ref{algo:expand}.

   \;$\bullet$\; 
 If the conditions for expanding $E_t$ are not satisfied, the algorithm evaluates the function $f$ at $x_t$, and updates $\mu_t, \sigma_t, n_e$ and $\mc{D}_t$.  
     
   \;$\bullet$\; 
   Finally,  when the budget $n$ is exhausted,  it recommends a point $z_n$ according to the rule \recommend~described in Definition~\ref{def:recommend} below. 


\begin{definition}[\recommend]
\label{def:recommend}
Suppose the algorithm stops in round $t_n$ and let $\Tau \subset \{1,2,\ldots, t_n\}$ denote the set of times at which \algoref{algo:algo1} performed function evaluations (note that $|\Tau|=n)$. Define $E_n = \argmin_{E \in \mc{P}_{t_n}} r_E$, and $\xi \coloneqq \min_{t \in \Tau} \beta_t \sigma_t(x_t)$. If $L (\sqrt{D}r_{E_n})^{\alpha_1} \leq \xi$, then return $\recc = x_{E_n}$, where $x_{E_n}$ is the center of the cell $E_n$. Else, return $\recc= x_{\tau}$ where $\tau \coloneqq \argmin_{t \in \Tau} \beta_t \sigma_t(x_t)$. 
\end{definition}

We next describe the details of the \expand~algorithm.

\subsection{\expand Algorithm}
\label{subsec:expand}

 The \expand algorithm (pseudo-code in Algorithm~\ref{algo:expand}) takes in as inputs a cell $E$, data $\mc{D}$, an integer $k$, confidence parameter $\delta>0$, a variable \texttt{flag} taking values in $\{1,2\}$,  a positive quantity \ttt{val} along with the terms $k, \alpha$,  $B$ and $L$.  It outputs   $\mc{P}_2$,  which is a partition of the cell $E$ constructed in accordance with the other input parameters, and assigns an upper bound on the value of $f$ in every element $F$ of $\mc{P}_2$. 
First, we   introduce the \partition~operation which takes a cell $E \subset \domain$ and a real number $r<r_E$ as inputs, 
and returns a partition of $E$ consisting of cells of side-length $r$. 

\begin{definition}[\partition] Given a cell $E = \bigtimes_{i=1}^D [a_i, b_i] \subset \domain$, the function call $\partition(E, r)$ for an some $r <r_E$ returns a partition of $E$ of cardinality $\lceil r_E/r\rceil^D$, consisting of sets of the form $F = \bigtimes_{i=1}^D [\tilde{a}_i, \min\{\tilde{a}_i + r, b_i\}]$, where   $\tilde{a}_i = a_i + lr$ for $l \in \{0,1,\ldots, \lfloor r_E/r\rfloor\}$. 
\label{def:partition}
\end{definition}

Note that when $r$ doesn't divide $r_E$, the partition returned by a call to \partition  contains some cells with sides strictly smaller than $r$. This does not affect our analysis as it only requires upper bounds on the cell side-lengths. 

\noindent We now describe the working of \expand~algorithm:

\;$\bullet$\;
\ttt{flag}=$1$ indicates that the call to \expand~was triggered either by the condition $\beta_n \sigma_t(x_t) \leq L(\sqrt{D}r_E)^{\alpha_1}$ in Line~\ref{algline:cond1} or by the condition $b_t(E_t) \leq L(\sqrt{D}r_E)^{\alpha_1}$ in Line~\ref{algline:cond2} of   Algorithm~\ref{algo:algo1}. In this case, the \expand~algorithm returns $\mc{P}_2$, the partition of  $E$ into cells of side $r_E/2$, and  sets $u_F^{(0)}$ equal to $u_{t,E}^{(1)}$ (Line~\ref{algline:cond1}) or $u_{t,E}^{(2)}$ (Line~\ref{algline:cond2}) for elements $F$ of $\mc{P}_2$. 



\;$\bullet$\;  \ttt{flag}=$2$ indicates that the call to \expand~was triggered by the condition on Line~\ref{algline:cond3} of \algoref{algo:algo1}. Under this condition,
 the algorithm first calls the \maxerr~function~(described in Def.~\ref{def:maxerr}), which itself calls the \localpoly~function~(Def.~\ref{def:localpoly}),  to get the uniform approximation error (\ttt{err}) using an appropriate local polynomial estimator for $f$ in the cell $E$. The algorithm then computes a value $\tilde{r}$ in Line~\ref{algline:length}, and then partitions $E$ into smaller cells of radius $\tilde{r}$ in Line~\ref{algline:partition} to construct $\mc{P}_2$. The length $\tilde{r}$ is chosen to ensure that for the cells of the partition, \ttt{err} is greater than the bound on the function variation $L(\sqrt{D}\tilde{r})^{\alpha_1}$. 
 Finally, for all cells $F$ in $\mc{P}_2$, the algorithm sets $u_F^{(0)}$ equal to the upper bound by calling the \localpoly~function~(described in Def.~\ref{def:localpoly}) and using the error bound \ttt{err}. 


To complete the description of the \expand algorithm, we next describe the contruction of local polynomial estimators, and introduce the functions \localpoly~(in Def.~\ref{def:localpoly}) and \maxerr~(in Def.~\ref{def:maxerr}) which are called by \expand. 

\begin{wrapfigure}{R}{0.45\textwidth}
\begin{minipage}{0.45\textwidth}
\begin{small}
\begin{algorithm}[H]
\SetAlgoLined
\SetKwInOut{Input}{Input}\SetKwInOut{Output}{Output}
\SetKwFunction{LocalPoly}{LocalPoly}
\SetKwFunction{Expand}{Expand}
\SetKwFunction{Partition}{Partition}
\newcommand\mycommfont[1]{\ttfamily\textcolor{blue}{#1}}
\SetCommentSty{mycommfont}

\KwIn{$E$, \texttt{flag},  \texttt{val}, $\mc{D}$, $k$, $\alpha$,  $\delta$}
 $n_E \leftarrow |\mc{D}^{(E)}|$ $\;$ \\
 \BlankLine
 \uIf{\texttt{flag}==1 
 }{
$\mc{P}_2 = \partition(E, r_E/2)$ \\
\For{$F \in \mc{P}_2$}{
$u_F^{(0)} = \texttt{val}$}
 }
 \Else{ 
 $ \texttt{err} = \maxerr(E, \mc{D}, K, B, k, \delta, \sigma) $\\ 
$\tilde{r} = \min\left \{ \frac{r_E}{2}, \,  \frac{1}{\sqrt{D}}\lp \frac{\ttt{err}}{L}\rp^{1/\alpha_1} \right \} \label{algoline:expand2} 
$ \label{algline:length}\\
$\mc{P}_2 = \partition(E, \tilde{r})$ \label{algline:partition}\\
\For{$F \in \mc{P}_2$}{
$\hat{f}_F = \localpoly(F, \mc{D}, x_F)$ \label{algoline:expand1} \\
$u_F^{(0)} = \hat{f}_F + 2 \texttt{err}$
}
 }
 \KwOut{$\mc{P}_2$}
 \caption{\expand~algorithm}
 \label{algo:expand}
\end{algorithm}
\end{small}
\end{minipage}
\end{wrapfigure}

\subsubsection{Local Polynomial~(LP) Estimators}
\label{sub2sec:localpoly}

Given a cell $E \subset \domain$ and a point $z \in E$, we define the LP estimator at $z$ as 
$
\hat{f}_E (z, \vec{w}) = \sum_{x  \in \mc{D}^{(E)}} w_{x} y_x, 
$
where the `interpolation weights' $\vec{w} = \{w_{x}\;:\;x \in \mc{D}_{\mc{X}}^{(E)}\}$ are defined as the solution of the following problem~\citep[Eq.~(1.36)]{nemirovski2000topics}:
\begin{small}
\begin{equation}
\begin{aligned}
    \min_{\vec{v} = \{v_x \,:\, x \in \mc{D}_{\domain}^{(E)}\} } \; \sum_{x \in \mc{D}^{(E)}_{\mc{X}}} |v_x|^2 \;\; 
    \text{s.t.} \; p(z) = \sum_{x \in \mc{D}_{\mc{X}}^{(E)}} v_x p(x)\quad \forall p \in \mc{P}_{D}^k. 
\end{aligned}
\tag{\tbf{LP}}
\label{LP_opt}
\end{equation}
\end{small}
If the number of data points in the cell $E$, i.e., $|\mc{D}_{\mc{X}}^{(E)}|$, is larger than $(k+2)^D$, then a unique solution to the problem~\eqref{LP_opt} is guaranteed to exist. We next present the definition of a function that will be used by our proposed algorithm. 

\begin{definition}[\localpoly]
\label{def:localpoly}
Given a cell $E$, a point $x \in E$, the function \localpoly~  returns the estimated function value $\hat{f}_E(x, \vec{w})$  at $x$, calculated according to the formula stated above. If $n_E \coloneqq |\mc{D}^{(E)}| > (k+2)^D$,  the weights $\vec{w}$ are the solution to~\eqref{LP_opt}, while if $n_E < (k+2)^D$,  the weights are set as $w_x = 1/n_E$ for all $x \in \DXE$.  
\end{definition}

\noindent We next state a result which bounds the estimation error between $\hat{f}_E(x, \vec{w})$ and $f(x)$.
\begin{lemma}{\citep[Prop.~1.3.1]{nemirovski2000topics}}
\label{lemma:LP1}
Given a labelled dataset $\mc{D}$, a cell $E \subset \domain$, and a point $x_E \in E$,  let $\vec{w}_E$ represent the  unique solution of the problem~\eqref{LP_opt} under the assumption that $n_E \coloneqq |\mc{D}^{(E)}| \geq (k+2)^D$.
Then, assuming that the observation noise is $\sigma-$subgaussian, for any $\delta>0$, we have  with probability at least $1-\delta$:
\begin{align}
   \vert \hat{f}_E(x_E, \vec{w}_E) - f(x_E) \vert &\;\leq\; \lp 1 + \|\vec{w}_E\|_1 \rp \Phi_k(f, E)  
   \; +\; \sigma \|\vec{w}_E\|_2 \sqrt{2 \log(2/\delta)}. \label{eq:bias1} 
\end{align}
\end{lemma}

\begin{remark}
\label{remark:lp_error_bounds}
Recall that term $\Phi_k(f, E)\coloneqq \inf_{f \in \mc{P}_D^k}\sup_{x \in E} |f(x)-p(x)|$ depends on how well elements of $\rkhs$ can be approximated by polynomials in $\poly$. For functions $f$ with $\|f\|_{\mc{C}^{k,\alpha}} \leq L$, it is known that we have $\Phi_k(f,E) \leq L (\sqrt{D}r_E)^{k+\alpha}$. 
Using this, we can construct an upper bound on the first term on the RHS of~\eqref{eq:bias1}, denoted by $e_D(E, \vec{w}_E, K, B) \coloneqq (1 + \|\vec{w}_E\|_1)L(\sqrt{D}r_E)^{k+\alpha}$, from the information available to the algorithm. The second term in the RHS of~\eqref{eq:bias1}, which we shall denote by $e_S(\vec{w}_E, \sigma, \delta)$,  depends only on the data and known terms~($\sigma$ and $\delta$) and thus can be computed as well.  
\end{remark}


The last definition required is the \maxerr~operation, which computes the maximum estimation error when a LP estimator is used to estimate $f$ in a cell $E$. 

\begin{definition}[\maxerr]  The function \maxerr takes $E$, $\mc{D}$, $K$, $B$,  $k$, $\delta$ and $\sigma$ as inputs, and returns \ttt{err} defined as 
\begin{equation*}
   \ttt{err} \coloneqq \max_{x \in E}\; e_D(E, \vec{w}_{E, x}, K, B) + e_S(\vec{w}_{E, x}, \sigma, \delta), 
   \label{eq:maxerr}
\end{equation*}
where $e_S$ and $e_D$ were introduced in Remark~\ref{remark:lp_error_bounds}, and $\vec{w}_{E,x}$ denotes the solution to~\eqref{LP_opt} at $x\in E$. 
\label{def:maxerr}
\end{definition}
Note that the objective function for the $\max$ operation in the above definition varies continuously with $x$, since it is known that the mapping $x \mapsto \vec{w}_{E,x}$ is continuous. This ensures that the maximum in the definition of \ttt{err} is achieved at some point in $E$. 

\subsection{Practical Issues and Heuristic}
\label{subsec:practical}

While the \lpgpucb algorithm achieves improved theoretical guarantees as we show next in Section~\ref{sec:regret_analysis}, they may not be reflected in practice through a naive implementation of \lpgpucb.
The reasons for this are \tbf{(1)}  our theoretical bounds hold only for $n$ \emph{large enough} (see Remark~\ref{remark:n_large_enough} in Appendix~\ref{appendix:regret}), \tbf{(2)} every partition update by the \expand algorithm adds at least $2^D-1$ new cells which may result in prohibitive memory costs, and \tbf{(3)} the LP estimators (Lines~8~to~13 of \expand) come into play only if $n_E>(k+2)^D$ which may not be satisfied in practice. These issues are further exacerbated by the fact that most practical applications of GP bandits are in the small $n$ regime due to the high cost associated with function evaluation. 

To mitigate these issues, we  propose a computationally efficient heuristic version of \lpgpucb, which fits a regression tree  to the data observed at the beginning of each round instead of employing the \expand algorithm. More specifically, the \ttt{Heuristic} algorithm:\\
\tbf{(1)} proceeds as the \lpgpucb algorithm with $k=0$ and $\alpha=1$,  and \\
\tbf{(2)} updates the partition $\mc{P}_t$ every step by fitting a regression tree with MSE criteria at the beginning of every round instead of employing the \expand algorithm. \\
While the theoretical analysis of the \heuristic algorithm~(pseudo-code in Appendix~\ref{appendix:heuristic}) is beyond the scope of this work, it is empirically shown to outperform several baseline methods in some experiments~(Sec.~\ref{sec:empirical}). 

\section{Regret Analysis of \lpgpucb}
\label{sec:regret_analysis}

We now state the main result of this section which provides high probability regret bounds for \lpgpucb algorithm. 
\begin{theorem} 
\label{theorem:regret1} 
Suppose Assumption~\ref{assump:main} holds, and Algorithm~\ref{algo:algo1} is run with a budget $n$, and other inputs as described in Section~\ref{sec:LPGPUCB}. Then the following statements are true with probability at least $1-\delta$, for a given $\delta \in (0,1)$. 
\begin{itemize}
\item We recover the following $\igain$ dependent bounds. 
\begin{align}
    \creg &= \tOh{ \igain \sqrt{n} } \quad \text{and} \quad \sreg = \tOh{ \igain/\sqrt{n}}. \label{eq:info_type1}
\end{align}
\item In addition, we also get the following smoothness dependent bounds for $n$ large enough\footnote{ the precise meaning of \emph{$n$ large enough} is described in Remark~\ref{remark:n_large_enough} in Appendix~\ref{appendix:regret}.}:
\begin{align}
    \sreg & = \tOh{n^{-\frac{(k+\alpha)}{D + 2(k+\alpha)}}}, \; \text{and } \; \creg = \tOh{ n^{\frac{2(k+\alpha)-\alpha_1 + D}{2(k+\alpha)+D}}}, \; \text{if } \igain =\Omega(\sqrt{n}), \label{eq:smooth1}\\
 \sreg & = \tOh{ n^{-\frac{\alpha_1}{D + 2\alpha_1}}}, \; \text{and} \; \creg = \tOh{n^{\frac{\alpha_1+D}{2\alpha_1 + D}}}, \; \text{otherwise.} \label{eq:smooth2}
\end{align}
\end{itemize}
In the above display, $\tilde{\mc{O}}$ hides the poly-logarithmic factors, and recall that $\alpha_1 = \max\{\alpha, \, \min\{1, k\}\}$. 
\end{theorem}
The proof of this statement is given in Appendix~\ref{appendix:regret}. 
\begin{remark}
Since the bounds given in~\eqref{eq:info_type1} and one of~\eqref{eq:smooth1}~or~\eqref{eq:smooth2} hold simultaneously under the $1-\delta$ probability event, the resulting bound by taking their minimum is always as good as the existing $\igain$ dependent results. In particular, since it is known from prior work that the $\igain$ dependent bounds on $\sreg$ and $\creg$ are near-optimal for SE kernels, \lpgpucb algorithm also matches that performance due to~\eqref{eq:info_type1}. 
However, as we show later, for the \matern~family of kernels, the smoothness dependent bounds are a strict improvement over the best known existing results. 
\end{remark}

An interesting special case of \lpgpucb algorithm is with the  parameter $k$ is equal to $0$, which will be useful in obtaining improved regret bounds for some kernels (Props.~\ref{prop:matern_lp0} and~\ref{prop:regret2}). We state the regret bounds for this case, as a corollary of Theorem~\ref{theorem:regret1}. 

\begin{corollary}
\label{corollary:regret2}
Suppose Assumption~\ref{assump:main} holds with $k=0$, and Algorithm~\ref{algo:algo1} is run with a budget $n$, and other inputs as described in Sec.~\ref{sec:LPGPUCB}. Then, for $\delta \in (0,1)$,  we have with probability at least $1-\delta$:
\begin{align}
    \sreg = \tOh{ \min \lp \igain/\sqrt{n}, \; n^{-\frac{\alpha}{D + 2\alpha}} \rp }, \; \text{ and }
    \creg = \tOh{ \min \lp \sqrt{n} \gamma_n, \; n^{ \frac{\alpha+D}{2\alpha+D}} \rp}. 
\end{align}
\end{corollary}

\subsection{Regret bounds for specific kernels}
\label{subsec:analysis}
In this section, we specialize the general regret bounds presented in Theorem~\ref{theorem:regret1} and Corollary~\ref{corollary:regret2} to some kernels used in practical machine learning applications.  These commonly used kernels include SE, \matern, RQ, GE and PP kernels.

\subsubsection{\matern kernels}
We begin the analysis with a key embedding result which says that we can identify elements of the RKHS associated with kernel $K \in (\kmat)_{\nu>0}$ with elements of certain \holder~spaces. Since the RKHS associated with $\kse$ kernel is included in the RKHS of $\kmat$ for any $\nu>0$, the result also holds for $\kse$.

\begin{proposition}
\label{prop:embed1}
If $f \in \rkhs$ with $K \in \{ \kse, \kmat\}$, then there exists  constants $0<C_1, C_2 < \infty$ such that we have $\|f\|_{\holdersp{\alpha}} \leq C_1C_2 \|f\|_{K}$ where $k \in \mbb{N}$ and $\alpha \in (0,1]$ such that $k+\alpha = \nu$. 
\end{proposition}

\begin{proofoutline}
The proof relies on two key observations: (i) the norm equivalence between the RKHS associated with $\kmat$ and certain fractional Sobolev Spaces $W^{\nu+D/2, 2}$ (defined in Appendix~\ref{appendix:embedding}), and (ii) the existence of a continuous embedding from $W^{\nu+D/2, 2}$ to the space $\holdersp{\alpha}$ with $(k,\alpha)$ given in the statement. 
Finally, the result for $\kse$ follows from the inclusion of the RKHS of $\kse$ in the RKHS of $\kmat$ for $\nu>0$.  The details are in Appendix~\ref{proof:prop_embed1}.
\end{proofoutline}

\begin{remark}
\label{remark:embed1}
Note that due to the nested nature of the \holder~function spaces, if $f\in \holdersp{\alpha}$ with $\|f\|_{\holdersp{\alpha}} \leq L$ for some $L>0$ and $k\geq 1$, then $f$ is also Lipschitz continuous with constant $L$. We will use this fact to obtain improved bounds on the cumulative regret for \matern~kernels in Proposition~\ref{prop:matern_lp0}. 
\end{remark}

We note that the constant $C_1$  in the proof outline of Prop.~\ref{prop:embed1} can be computed in terms of the parameters $\nu$ and $D$. Appropriate bounds on the other term $C_2$ also exist for some cases, see for example~\citep{talenti1976best}. For simplicity however, we make the following assumption. 
\begin{assumption}
\label{assump:embed}
We assume that $n$ is large enough to ensure that $C_2 \leq \log(n)$ for   $K=\kmat$  where $C_2$ is the constant introduced in Proposition~\ref{prop:embed1}.  
\end{assumption}
We can now state the regret bounds for \matern~kernels as a special case of Theorem~\ref{theorem:regret1}. 
\begin{proposition}
\label{prop:matern_lpk}
Suppose Assumptions~\ref{assump:main}~and~\ref{assump:embed} hold, and Algorithm~\ref{algo:algo1} is run with budget $n$, $k = \lceil \nu \rceil -1$, $\alpha=\nu-k$, $L=BC_1\log n$ and other inputs as described in Sec.~\ref{sec:LPGPUCB}. Then the following statement is true with probability at least $1-\delta$, for $n$ large enough:
\begin{align}
\label{eq:reg_matern1}
    \sreg = \tOh{ \min \lp n^{\frac{1}{2} - \frac{D(D+1)}{2\nu + D(D+1)}}, \; n^{-a_{\nu}} \rp}, \quad \text{and} \quad \creg = \tOh{ \min \lp n^{\frac{1}{2} + \frac{D(D+1)}{2\nu + D(D+1)}}, \; n^{-b_{\nu}}  \rp}. 
\end{align}
 The exponent $a_{\nu}$ is $\alpha_1/(2\alpha_1+D)$ for $\nu > D(D+1)/2$ and $\nu/(2\nu+D)$ otherwise, while the exponent $b_{\nu}$ is $(D+\alpha_1)/(D+2\alpha_1)$ for $\nu>D(D+1)/2$ and $(2\nu-\alpha_1 + D)/(2\nu+D)$ otherwise. Recall that $\alpha_1 = \max \{ \alpha, \min \{1, k\}\}$. 
\end{proposition}

The above result follows from a combination of Theorem~\ref{theorem:regret1} with the embedding result of Proposition~\ref{prop:embed1} along with Assumption~\ref{assump:embed} and employs the bounds on $\igain$ derived by \citep{srinivas2012information}. Note that under the condition $\nu \leq D(D+1)$,  the upper bound on $\igain$  is $\Omega(\sqrt{n})$. In this parameter range,  the \lpgpucb algorithm achieves  near-optimal rates for $\sreg$, thus closing the large gap between the upper and lower bounds in the literature. However, the improvement achieved for $\creg$ is not as significant as that of $\sreg$. The main reason is that for  $k \geq 1$, the \lpgpucb algorithm is more \emph{exploratory} as it performs many function evaluations before expanding cells with radius smaller than $\rho_0$ (see Lemma~\ref{lemma:smooth3} in Appendix~\ref{appendix:regret} for precise statement). While this ensures that the algorithm can find at least one \emph{good} point resulting in small $\sreg$, due to insufficient \emph{exploitation},  the bound on $\creg$ suffers. A similar trade-off between obtaining tight bounds for both $\sreg$ and $\creg$ occurs in some other bandit problems as well; see for example~\citep[\S~3]{bubeck2011pure}. 
In Proposition~\ref{prop:matern_lp0} next, we show that the by using a zeroth degree LP estimators, we can obtain tighter control over $\creg$ for \matern~kernels with $\nu>1$.

\begin{proposition}
\label{prop:matern_lp0}
For $f \in \rkhs[\kmat]$ with $\nu>1$, the \lpgpucb algorithm with $k=0$ achieves the following regret bounds:
\begin{align}
    \sreg = \tOh{\min \lp n^{-1/2 + c_\nu}, n^{-1/(D+2)} \rp}, \qquad \text{and }\quad \creg = \tOh{ \min \lp n^{1/2 + c_\nu}, n^{(D+1)/(D+2)}\rp}, \label{eq:matern_reg}
\end{align}
where $c_\nu = \frac{ D(D+3)}{4\nu + D(D+5)}$. 
\end{proposition}
\begin{proof}
From Remark~\ref{remark:embed1}, we know that $f\in \rkhs[\kmat]$ with $\|f\|_{\kmat}\leq B$ for $\nu>1$ are Lipschitz continuous, and hence the bounds or Corollary~\ref{corollary:regret2} hold with $\alpha=1$. Furthermore, by Lemma~\ref{lemma:smooth2} we know that each cell $E$ is evaluated no more than $\tOh{r_E^2}$ times. This fact allows us to apply bounds on the information gain for \matern~kernels for hierarchical sampling algorithms derived in \citep[Theorem~3]{shekhar2019multiscale} to get $\igain = n^{c_\nu}$ with $c_\nu = \frac{ D(+3}{4\nu + D(D+5)}$. 
\end{proof}
Combined with the regret bounds for \matern~kernels derived in Theorem~\ref{theorem:regret1}, the above result implies that the \lpgpucb algorithm (with appropriate choice of $k$) achieves improved regret bounds for the \matern~family of kernels for all values of $\nu$ and $D$. Further discussion is in Sec.~\ref{subsec:improved}.

\subsubsection{Regret Bounds for other kernels}
\label{subsec:lp0}
We begin with the following embedding result. 
\begin{proposition}
\label{prop:embed2}
If $f \in \rkhs$ for $K \in  \mc{K} \coloneqq  \{\krq,\; \kgexp,\; \kpp \}$, then we have $\|f\|_{\mc{C}^{0, \alpha}} \leq \sqrt{2}C_K \|f\|_{K}$ with $\alpha \in \{1/2, 1\}$ and $C_K$ is a constant that can be computed from the knowledge of $K$. 
\end{proposition}
\begin{proofoutline}
The proof of this statement proceeds by first using the reproducing property of $\rkhs$ and the Cauchy-Schwarz inequality to obtain $|f(x_1)-f(x_2)| \leq \sqrt{2(K(0) - K(\|x_1-x_2\|) )}$ for $x_1,x_2 \in \domain$. To complete the proof we then have to perform some computations specific to each kernel in $\mc{K}$ to obtain the required values of $\alpha$ and $C_K$. The details are in Appendix~\ref{proof:prop_embed2}. 
\end{proofoutline}




By employing the result of Proposition~\ref{prop:embed2} with the regret bound in Corollary~\ref{corollary:regret2}, we get the following. 
\begin{proposition}
\label{prop:regret2}
Suppose Assumption~\ref{assump:main} holds, and Algorithm~\ref{algo:algo1} is run with budget $n$, $K \in \mc{K}$,  $L=\sqrt{2}C_KB$, $k=0$, $\alpha$ according to Proposition~\ref{prop:embed2} and other inputs as described in Sec.~\ref{sec:LPGPUCB}. Then we have with probability at least $1-\delta$, 
\begin{align*}
    \sreg = \tOh { \min \lp \gamma_n/\sqrt{n}, \; n^{-\alpha/(2\alpha+D)} \rp }, 
    \quad \text{and} \quad
    \creg &= \tOh{ \min \lp \sqrt{n} \gamma_n, \; n^{(\alpha+D)/(2\alpha+D)} \rp}.
\end{align*}
\end{proposition}

Note that since $n$ dependent bounds on $\igain$ are not known for kernels in the set $\mc{K}$, the result of Proposition~\ref{prop:regret2} provides us with the first explicit regret bounds for these kernels. 


\subsection{Summary of Improvements}
\label{subsec:improved}
We now summarize the regret bounds achieved by Algorithm~\ref{algo:algo1} for different kernels. 

\vspace{0.5em}
\noindent\tbf{SE kernel.} The $\igain$ dependent bounds with the poly-logarithmic upper bounds on $\igain$ derived by \citep{srinivas2012information} are known to be near-optimal for SE kernel. The \lpgpucb algorithm also matches these results. 

\vspace{0.5em}
 \noindent\tbf{\matern~kernels.} To discuss the results for \matern~kernels, we introduce the notations  $\mc{I}_0 = (0,1]$, $\mc{I}_1 = (1, D(D+1)/2]$, $\mc{I}_2 = ( \small{D(D+1)/2}, \small{(D^2 + 5D+12)/4}]$ and $\mc{I}_3 = (e, \infty)$ where $e = \max\{ D(D+1)/2, (D^2 + 5D + 12)/4\}$. Note that $\mc{I}_2$ is non-empty only for $D \leq 5$. We have the following:

\begin{itemize}[leftmargin=*]
\item \emph{Simple Regret.} The best bounds on $\sreg$ is $\tOh{n^{-\nu/(2\nu+D)}}$ for $\nu \in \mc{I}_0\cup \mc{I}_1$ and $\tOh{ n^{-1/(D+2)}}$ for $\nu \in \mc{I}_2$, both of which are achieved by the \lpgpucb algorithm with $k = \lceil \nu \rceil-1$, i.e.,~\eqref{eq:reg_matern1}. 
For $\nu \in \mc{I}_3$, we have $\sreg = \tOh{n^{-\frac{1}{2} + c_\nu}}$ for $c_\nu = \frac{D(D+3)}{4\nu+D(D+5)}$ which is achieved by the algorithm with $k=0$. Note that for $\nu \in \mc{I}_0\cup \mc{I}_1$, the regret bound achieved by \lpgpucb algorithm is near-optimal, i.e., it matches the lower bound up to poly-log factors.  This closes the large gap between the upper and lower bounds for existing  algorithms, as in this parameter regime the existing bounds on $\sreg$ are $\Oh{1}$ due to $\igain$ being $\Omega(\sqrt{n})$. 

\item \emph{Cumulative Regret.} The \lpgpucb algorithm with $k=0$ achieves tighter control over $\creg$ for all values of $\nu$. In particular for $\nu \in \mc{I}_0$ we have $\creg = \tOh{n^{(\nu+D)/(2\nu+D)}}$, which matches the lower bound of \citep{scarlett2017lower} up to poly-log factors. For   $\nu \in \mc{I}_1 \cup \mc{I}_2$, $\creg$ is $\tOh{n^{(D+1)/(D+2)}}$  while for $\nu\in \mc{I}_3$ it is $\tOh{n^{1/2 + c_\nu}}$ with $c_{\nu}$ defined in Prop.~\ref{prop:matern_lp0}. 

The cumulative regret bounds discussed above improve upon the state-of-the-art for all values of $\nu, D$. More specifically, \tbf{(i)} since $\igain$ based bounds are known only for $\nu>1$, our results provide the first explicit (and near-optimal) regret bounds of \matern~kernels with $\nu\in \mc{I}_0=(0,1]$, and \tbf{(ii)} for $\nu>1$ our results improve upon the best known bounds of $\tOh{n^{e_{\nu}}}$ with $e_{\nu} = \frac{ D(2D+3) + 2\nu}{D(2D+4)+4\nu}$, recently derived by \citep{janz2020bandit}. 

\end{itemize}

 \noindent \tbf{Other kernels.} 
 For several other important kernels, including GE, RQ and PP, we derive the first explicit (in $n$) bounds on $\sreg = \tOh{n^{-\alpha/(D+2\alpha)}}$  and $\creg = \tOh{n^{(D+\alpha)/(D+2\alpha)}}$ with $\alpha \in \{1/2, 1\}$ depending on the kernel. Since there do not exist any algorithm-independent lower bounds for these kernels, it is not clear how sub-optimal these bounds are.

\section{Empirical Results}
\label{sec:empirical}
In this section, we empirically test the performance of our proposed algorithm \lpgpucb and \ttt{Heuristic} against some well known benchmarks on synthetic functions as well as a hyperparameter tuning task. In all the experiments, we use the \matern kernel $K_\nu$ with $\nu=2.5$, and set $B=1$ and $L=\sqrt{2}$~(where applicable), $\sigma=0.1$ and $\delta=0.001$. 

\tbf{Algorithms.} We used the following algorithms:
\tbf{(1)} \emph{LP0}: This is the simplest version of our \lpgpucb algorithm, which uses $k=0$ and $\alpha=1$, 
\tbf{(2)} \emph{Heuristic}, introduced in Sec.~\ref{subsec:practical} with details in Appendix~\ref{appendix:heuristic}, 
\tbf{(3)} \emph{IGPUCB}, the improved GP-UCB algorithm of \citep{chowdhury2017kernelized}, 
\tbf{(4)} \emph{EI}, Expected Improvement, and \tbf{(5)} \emph{PI}, Probability of Improvement. 

\tbf{Expt. 1: Synthetic functions.} The goal of the first experiment is to test if the algorithms can detect certain underlying structure in the objective functions. More specifically, we construct an $f:\reals^8 \mapsto \reals$ by using a benchmark function $g:\reals^2 \mapsto \reals$ as follows: $f(x) = \sum_{i=1}^4 c_i g(x[2i-1:2i])$ where $c_1=1$ and $c_i=0.1$ for $i=2,3,4$. For the 2-dimensional benchmark function $g$, we used Branin and Goldstein functions. Figure~\ref{fig:expt1} plots the best function value vs the number of evaluations for the five algorithms mentioned above. As we can see, due to the adaptive partitioning approach of  \heuristic and \lpgpucb, these algorithms were better able to exploit the simple structure in the objective function and find higher value points. 

\begin{figure}[t]
\captionsetup[subfigure]{labelformat=empty}
     \centering
    \subcaptionbox{
    \label{fig:branin}}{\includegraphics[width=0.48\columnwidth]{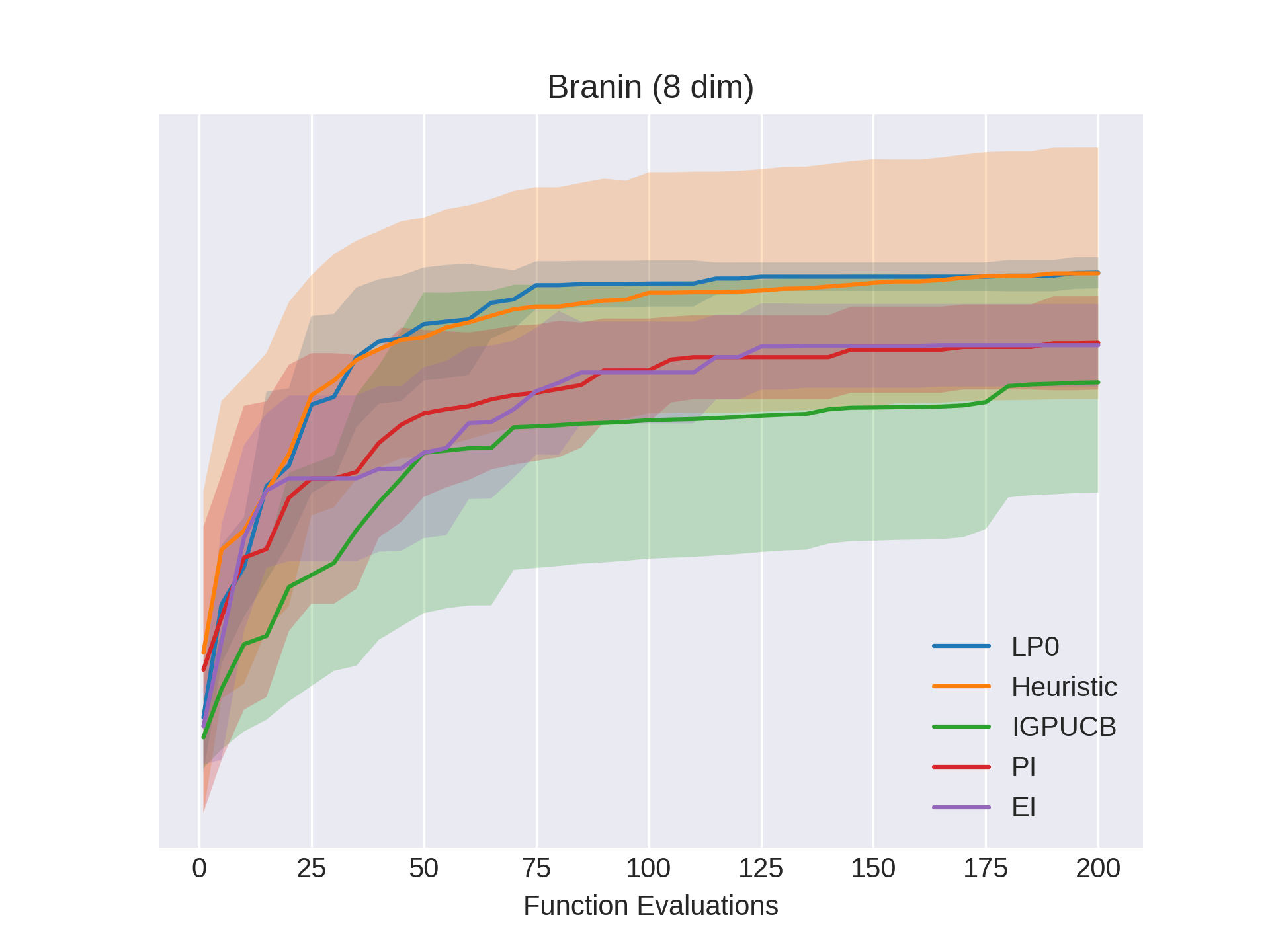}}
        \subcaptionbox{
        \label{fig:goldstein}}{\includegraphics[width=0.48\columnwidth]{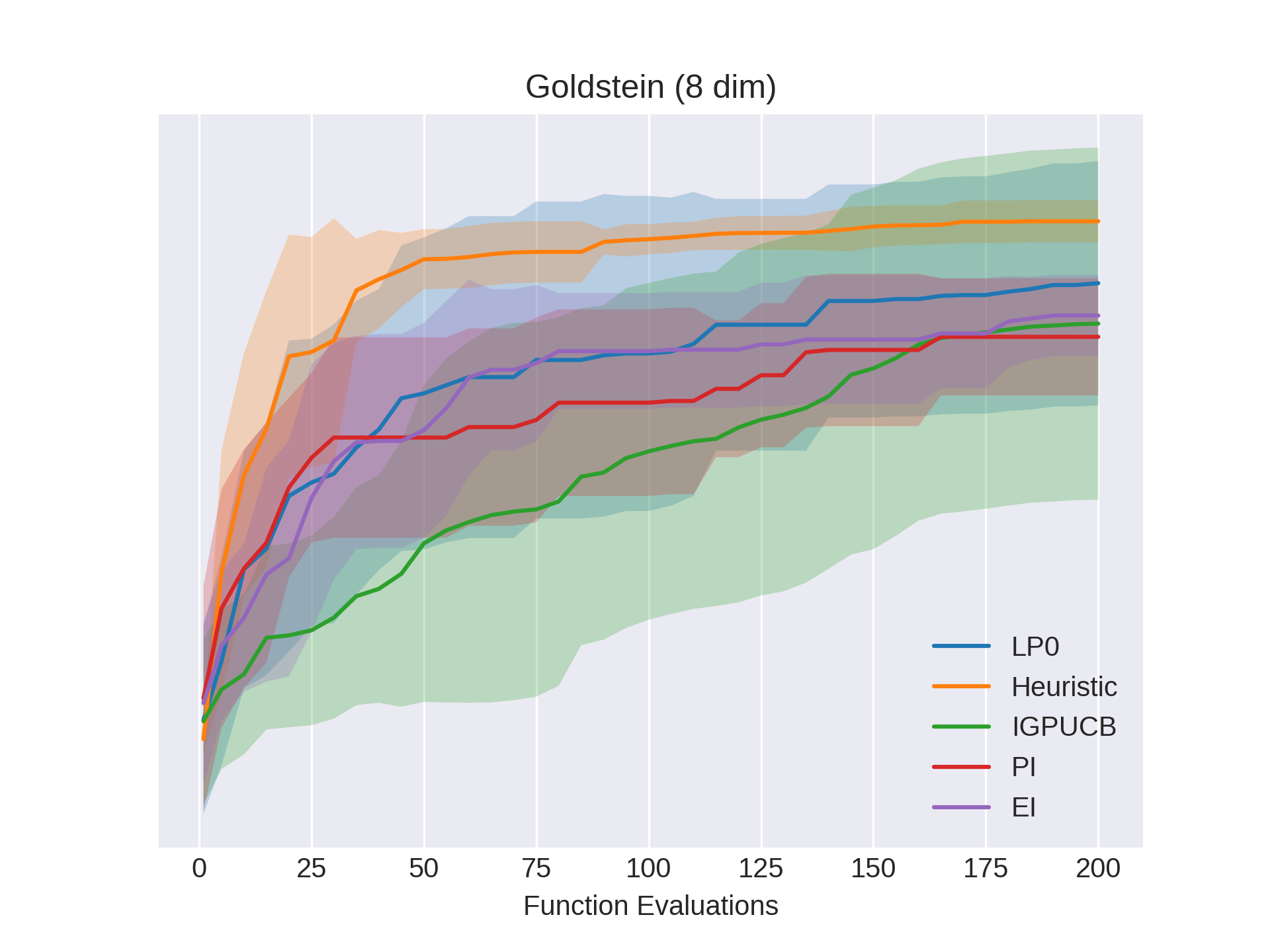}} 

    \vspace{-1em}
       \caption{Performance of the five algorithms on the task of optimizing synthetic 8-dim functions with underlying low-dimensional structure. }
        \label{fig:expt1}
\end{figure}

\tbf{Expt. 2: Hyperparameter Tuning.} In this experiment, we use the optimization algorithms for selecting the best hyperparameters of a Convolutional Neural Network (CNN) with two convolutional layers and two fully connected layers. The hyperparameter to be optimized were \ttt{batch\_size}, the \ttt{learning\_rate}, the \ttt{kernel\_size} of the two conv layers and the \ttt{hidden\_nodes} in the first fully-connected layer. The objective to be maximized was the accuracy on the test set. Further details of the experiment setup are given in Appendix~\ref{appendix:experiments}.  The mean accuracy of the point recommended by the algorithms over $10$ trials (each with a budget of $n=50$ evaluations) is shown in Table~\ref{table:mnist}. As indicated by the values, both LP0 and \heuristic achieve high classification accuracy with relatively small variability.

\begin{table}[!ht]
\centering
\begin{tabular}{| c c c |} \hline 
Method & Mean Accuracy & Std. Deviation \\ \hline 
LP0 & 90.117 & 1.479 \\ 
\heuristic & 90.777 & 1.711 \\
IGPUCB & 85.027 & 10.981 \\
EI & 88.892 & 1.679 \\
PI & 86.692 & 2.966 \\ \hline 
\end{tabular}
\caption{The table shows the performance of the five algorithms for the hyperparameter tuning task over 10 trials. Each trial consisted of $50$ iterations (i.e., $n=50$), and the table reports the mean and standard deviation of the test accuracy at the point (i.e., the configuration of hyperparmeters) recommended by the algorithms at the end of each trial.}
\label{table:mnist}
\end{table}



\section{Conclusion and Future Work}
\label{sec:conclusion}
In this paper, we proposed a new algorithm, \lpgpucb, for the problem of agnostic Gaussian Process bandits, and obtained high probability bounds on its simple and cumulative regret. For the practically useful \matern~family of kernels $(\kmat)_{\nu>0}$, we derive regret bounds which are tighter than existing bounds for all $\nu>0$, and is near-optimal in certain ranges of $\nu$. Furthermore, for this algorithm, we also obtained the first explicit regret bounds for some important kernels such as rational-quadratic, gamma-exponential and piecewise-polynomial kernels. Experimental evaluation on some benchmark functions as well as on a hyperparameter tuning task suggest that the proposed multi-scale partitioning approach may also be adapted for practical problems.

Our work opens several interesting directions for future research: 
\tbf{(1)} 
improving the cubic computational complexity associated with the exact GP inference in the implementation of \lpgpucb and \heuristic by using techniques such as adaptive sketching~\citep{calandriello2019gaussian}, 
\tbf{(2)} extension of the algorithmic techniques of this paper to related topics such as contextual GP bandits, GP level set estimation and parallel GP bandits. Since the existing theoretical results for these problems also depend on $\igain$, the methods of our paper may potentially lead to significant improvements in these problems as well. 


\newpage 
\bibliographystyle{apalike}
\bibliography{ref}

\begin{thebibliography}{}

\bibitem[Auer et~al., 2002]{auer2002finite}
Auer, P., Cesa-Bianchi, N., and Fischer, P. (2002).
\newblock {Finite-time Analysis of the Multiarmed Bandit Problem}.
\newblock {\em Machine learning}, 47(2-3):235--256.

\bibitem[Bogunovic et~al., 2018]{bogunovic2018adversarially}
Bogunovic, I., Scarlett, J., Jegelka, S., and Cevher, V. (2018).
\newblock {Adversarially Robust Optimization with Gaussian Processes}.
\newblock In {\em Advances in Neural Information Processing Systems}, pages
  5760--5770.

\bibitem[Bubeck et~al., 2011]{bubeck2011pure}
Bubeck, S., Munos, R., and Stoltz, G. (2011).
\newblock Pure exploration in finitely-armed and continuous-armed bandits.
\newblock {\em Theoretical Computer Science}, 412(19):1832--1852.

\bibitem[Calandriello et~al., 2019]{calandriello2019gaussian}
Calandriello, D., Carratino, L., Lazaric, A., Valko, M., and Rosasco, L.
  (2019).
\newblock Gaussian process optimization with adaptive sketching: Scalable and
  no regret.
\newblock {\em arXiv preprint arXiv:1903.05594}.

\bibitem[Chowdhury and Gopalan, 2017]{chowdhury2017kernelized}
Chowdhury, S.~R. and Gopalan, A. (2017).
\newblock On kernelized multi-armed bandits.
\newblock In {\em Proceedings of the 34th International Conference on Machine
  Learning-Volume 70}, pages 844--853. JMLR. org.

\bibitem[Contal et~al., 2013]{contal2013parallel}
Contal, E., Buffoni, D., Robicquet, A., and Vayatis, N. (2013).
\newblock {Parallel Gaussian Process Optimization with Upper Confidence Bound
  and Pure Exploration}.
\newblock In {\em Joint European Conference on Machine Learning and Knowledge
  Discovery in Databases}, pages 225--240. Springer.

\bibitem[Desautels et~al., 2014]{desautels2014parallelizing}
Desautels, T., Krause, A., and Burdick, J.~W. (2014).
\newblock {Parallelizing Exploration-Exploitation Trade-offs in Gaussian
  Process Bandit Optimization}.
\newblock {\em The Journal of Machine Learning Research}, 15(1):3873--3923.

\bibitem[Janz et~al., 2020]{janz2020bandit}
Janz, D., Burt, D.~R., and Gonz{\'a}lez, J. (2020).
\newblock {Bandit optimisation of functions in the Mat\'ern kernel RKHS}.
\newblock {\em arXiv preprint arXiv:2001.10396}.

\bibitem[Kandasamy et~al., 2016]{kandasamy2016gaussian}
Kandasamy, K., Dasarathy, G., Oliva, J.~B., Schneider, J., and P{\'o}czos, B.
  (2016).
\newblock {Gaussian process bandit optimisation with multi-fidelity
  evaluations}.
\newblock In {\em Advances in Neural Information Processing Systems}, pages
  992--1000.

\bibitem[Kandasamy et~al., 2015]{kandasamy2015high}
Kandasamy, K., Schneider, J., and P{\'o}czos, B. (2015).
\newblock {High dimensional Bayesian Optimisation and Bandits via Additive
  Models}.
\newblock In {\em International Conference on Machine Learning}, pages
  295--304.

\bibitem[Krause and Ong, 2011]{krause2011contextual}
Krause, A. and Ong, C.~S. (2011).
\newblock {Contextual Gaussian Process Bandit Optimization}.
\newblock In {\em Advances in neural information processing systems}, pages
  2447--2455.

\bibitem[Nemirovski, 2000]{nemirovski2000topics}
Nemirovski, A. (2000).
\newblock {Topics in Non-parametric Statistics}.
\newblock {\em Ecole d’Et{\'e} de Probabilit{\'e}s de Saint-Flour}, 28:85.

\bibitem[Salo, 2008]{salo2008lecture}
Salo, M. (2008).
\newblock {Function Spaces, Lecture Notes, Fall 2008}.
\newblock \url{http://users.jyu.fi/~salomi/lecturenotes/fsp08_lectures.pdf}.
\newblock [Online; accessed 25-March-2020].

\bibitem[Scarlett et~al., 2017]{scarlett2017lower}
Scarlett, J., Bogunovic, I., and Cevher, V. (2017).
\newblock Lower bounds on regret for noisy gaussian process bandit
  optimization.
\newblock In {\em Conference on Learning Theory}, pages 1723--1742.

\bibitem[Shekhar and Javidi, 2019]{shekhar2019multiscale}
Shekhar, S. and Javidi, T. (2019).
\newblock Multiscale gaussian process level set estimation.
\newblock In {\em The 22nd International Conference on Artificial Intelligence
  and Statistics}, pages 3283--3291.

\bibitem[Srinivas et~al., 2012]{srinivas2012information}
Srinivas, N., Krause, A., Kakade, S.~M., and Seeger, M.~W. (2012).
\newblock {Information-theoretic Regret Bounds for Gaussian Process
  Optimization in the Bandit Setting}.
\newblock {\em IEEE Transactions on Information Theory}, 58(5):3250--3265.

\bibitem[Talenti, 1976]{talenti1976best}
Talenti, G. (1976).
\newblock {Best constant in Sobolev inequality}.
\newblock {\em Annali di Matematica pura ed Applicata}, 110(1):353--372.

\bibitem[Wendland, 2004]{wendland2004scattered}
Wendland, H. (2004).
\newblock {\em {Scattered data approximation}}, volume~17.
\newblock Cambridge university press.

\bibitem[Williams and Rasmussen, 2006]{williams2006gaussian}
Williams, C.~K. and Rasmussen, C.~E. (2006).
\newblock {\em {Gaussian processes for machine learning}}, volume~2.
\newblock MIT press Cambridge, MA.

\end{thebibliography}

\newpage 
\onecolumn
\begin{appendix}
\section{Preliminaries}
\label{appendix:preliminaries}

A function $K: \domain \times \domain \mapsto \mbb{R}$ is called a positive-definite \emph{kernel} if for all $m \in \mbb{N}$ and $a_i \in \reals$ for $1 \leq i \leq m$, the following inequality holds: $\sum_{i=1}^m\sum_{j=1}^m a_i K(x_i, x_j) a_j \geq 0$. 
Every kernel function, $K$,  can be used to define a collection of functions $\{K(\cdot, x): \domain \mapsto \reals, \; x \in \domain\}$ and the corresponding linear span  $\mc{G}_K \coloneqq \{ \sum_{i=1}^m a_i K(\cdot, x_i)\; : \; m \in \mbb{N}, \; (x_i)_{i=1}^m \subset \domain, (a_i)_{i=1}^m \subset \reals\}$.  Using the positive-definiteness of $K$, the following inner-product can be defined on $\mc{G}_K$, $\langle f, g \rangle_{K} = \sum_{i=1}^{m_1} \sum_{j=1}^{m_2}a_i  b_j K(x_i, z_j)$ for $f=\sum_{i=1}^{m_1} a_i K(\cdot, x_i)$ and $g=\sum_{j=1}^{m_2} b_j K(\cdot, z_j)$. Thus $\lp \mc{G}_K, \langle \cdot, \cdot \rangle_{K} \rp$ is an inner-product space, which may not be complete. 
The \emph{Reproducing Kernel Hilbert Space} associated with kernel $K$, denoted by $\rkhs$,  is defined as the completion of the inner product space $\lp \mc{G}_K, \langle \cdot, \cdot \rangle_{K} \rp$. 

Two of the most commonly used kernels in machine learning are the Squared-Exponential~(SE) kernel $\kse$ and the \matern family of kernels, $(\kmat)_{\nu>0}$, parameterized by a smoothness parameter $\nu$. For $x, z \in \domain$ such that $\|x-z\|=r$, they are defined as 
\begin{align*}
    &\kse(x, z) = \kse (r) = \exp \lp - \frac{ r^2}{2 \theta_l^2} \rp,  \qquad 
    \kmat(x, z) = \kmat (r) = \frac{2^{1-\nu}}{\Gamma(\nu)} \lp \frac{ \sqrt{2 \nu }r}{\theta_l} \rp^{\nu} J_\nu \lp \frac{ \sqrt{2 \nu }r}{\theta_l} \rp, 
\end{align*}
where $\Gamma(\cdot)$ is the Gamma function,  $J_{\nu}$ represents the modified Bessel's function of the second kind, and $\theta_l>0$ is the length-scale parameter. 
In addition to the above two kernels, some other important kernels used in applications are Rational-Quadratic kernels $\krq$, Gamma-exponential kernels $\kgexp$ and piece-wise polynomial kernels. A detailed description of these kernel functions can be found in \citep[Chapter~4]{williams2006gaussian}.

For $0<\alpha\leq 1$, we use $\mc{C}^{0,\alpha}(\domain)$ to denote the the space of \holder~continuous functions $g:\domain \mapsto \mbb{R}$, which satisfy $[g]_\alpha < \infty$ where $[\cdot]_\alpha$ is defined in~\eqref{eq:holder}. Furthermore for $k \geq 1$, we can define higher order \holder~spaces $\mc{C}^{k,\alpha}$, as the collection of functions $g$ with $\|g\|_{\mc{C}^{k,\alpha}}<\infty$. In the definition of $\|\cdot\|_{\mc{C}^{k,\alpha}}$ in~\eqref{eq:holder},  we have $\|g\|_u \coloneqq \sup_{x\in \domain} |g(x)|$ and $\partial^a g$ denotes the derivative of $g$ with multi-index $a = (a_1, a_2, \ldots, a_D)\in\mbb{N}^D$ and $|a|\coloneqq \sum_{i=1}^Da_i$. 
\begin{align}
\label{eq:holder}
    [g]_\alpha \coloneqq \sup_{ x_1, x_2 \in \domain} \; \frac{ |g(x_1) - g(x_2)|}{\|x_1-x_2\|^\alpha}, \quad 
     \|g\|_{\mc{C}^{k,\alpha}} \coloneqq \sum_{|a|\leq k} \|\partial^ag\|_u +  \sum_{|a|=k} [\partial^ag]_{\alpha}.
\end{align}

\section{Proof of Propositions~\ref{prop:embed1} and~\ref{prop:embed2}}
\label{appendix:embedding} 
In this section, we provide the proofs of the  embedding results, Propositions~\ref{prop:embed1} and~\ref{prop:embed2}.

\subsection{Proof of Proposition~\ref{prop:embed1}}
\label{proof:prop_embed1}
 As mentioned in the proof outline, this result is obtained in two steps: (1) use the norm equivalence between the $\rkhs[\kmat]$ and fractional Sobolev spaces $\sobolev$, and (2) employ a Sobolev Embedding Theorem to then identify elements of $\sobolev$ with elements of \holder~space $\holderspk$. We elaborate on these steps next.  
 

First we recall the definition of Sobolev Spaces $W^{s,2}$ for $s>0$ as follows:
\begin{align*}
    W^{s,2} = \{f \in L^2(\domain)\;:\; \mc{F}f(\cdot)\lp 1 + \|\cdot\|^2 \rp^{s/2}  \in L^2(\domain)\}, 
\end{align*}
where $\mc{F}f$ is the Fourier Transform of $f$. Next, by Theorem~10.12 of \citep{wendland2004scattered}, we note that $\rkhs[\kmat]$  can also be defined as 
\begin{align*}
    \rkhs[\kmat] = \{f \in L^2(\domain)\;:\; \frac{ \mc{F}f}{\sqrt{\mc{F}\kmat}} \in L^2(\domain)\}.
\end{align*}
Finally, by noting that there exists a constant $C_1>0$ such that we have $\mc{F}\kmat(\omega) \geq \frac{1}{C_1 (1+\|\omega\|^2)^{\nu+D/2}}$, we complete the first step of the proof for \matern kernels, by showing that $\|f\|_{\sobolev} \leq C_1 \|f\|_{\rkhs[\kmat]}$. 

The final result then follows from an application of a Sobolev Embedding Theorem which says that there exists a constant $C_2 >0$, such that $\|f\|_{\holderspk} \leq C_2 \|f\|_{\sobolev}$ for $f \in \sobolev$. The specific form we use is given in Theorem~2 in Sec.~3.6 of~\citep{salo2008lecture}

Finally, for the case of SE kernel, we note  that for any $\nu>0$,  there exists a constant $C_{3,\nu}>0$ such that we have for all $\omega$ $\sqrt{\mc{F}\kmat(\omega)/\mc{F}\kse(\omega)} \geq 1/C_{3,\nu}$. This implies the following:
\begin{align*}
    \|f\|_{\kse} = \|\mc{F}f/\sqrt{\mc{F}\kse}\|_2 = \left \lVert   \frac{\mc{F}f}{\sqrt{\mc{F}\kmat}} \sqrt{\frac{\mc{F}\kmat}{\mc{F}\kse}}  \right\rVert_2 \geq \frac{1}{C_{3,\nu}} \|f\|_{\kmat}, 
\end{align*}
which implies the result for $\kse$.  \qed
\subsection{Proof of Proposition~\ref{prop:embed2}} 
\label{proof:prop_embed2}
First, we obtain a general bound on the term $|f(x) - f(z)|$ for any $x, z\in \domain$ with $\|x-z\|=r$ as follows:
\begin{align*}
     |f(x) - f(z)| & \stackrel{(a)}{=}  |\langle f, K_x - K_z \rangle| 
    \stackrel{(b)}{\leq}   \|f\|_K \|K_x - K_z\|_K \\
    & \stackrel{(c)}{\leq} B  \|K_x - K_z \|_K \stackrel{(d)}{=} \sqrt{2 \lp K(0) - K(r) \rp}. 
\end{align*}
In the above display, \tbf{(a)} uses the reproducing property of the RKHS, i.e., $f(x) = \langle f, K_x \rangle$, \tbf{(b)} uses the Cauchy-Schwarz inequality,  \tbf{(c)} follows from the bound on the norm of $\f$ according to Assumption~\ref{assump:main}, and \tbf{(d)} uses the fact that $K$ is isotropic for all $K \in \mc{K}$.


To complete the proof, we need to show that $\sqrt{K(0) - K(r)} \leq C_K r^{\alpha}$ for some $C_K >0$ and $\alpha \in \{1/2,1\}$ for all $K \in \mc{K}$. We consider the three kernels separately:

\tbf{RQ kernels.} For $x,z$ such that $\|x-z\| = r$, the RQ kernel with parameters $a,\theta_l$ is defined as 
\begin{align*}
    \krq(x,z) = \krq(r) = \lp 1 + \frac{r^2}{2a \theta_l} \rp^{-a}, \quad a,\theta_l >0.
\end{align*}
Now, using the fact that the function $g(x) \coloneqq (1+cx)^{-a}$ for $c>0$ is convex  in $x>0$, and the fact that for $x,z \in \domain$, we must have $\|x-z\|^2 \leq D$, we have:
\begin{align*}
    g(x) &\leq g(0) + \frac{g(D) - g(0)}{D-0} x     \Rightarrow g(x) -g(0)  \leq \frac{g(D)}{D}x 
    \Rightarrow \sqrt{g(x) - g(0)}  \leq \frac{r}{\sqrt{(1+cD)^{a}D}}. 
\end{align*}
Thus for the case of $\krq$, the required result holds with $\alpha=1$ and $C_{\krq} = \sqrt{\frac{1}{D (1+1/(2a\theta_l))^a}}$. 

\tbf{GE Kernels.} For $\theta_l>0$ and $0<a\leq 2$, and for $x, y \in \domain$ with $\|x-y\|=r$, the GE kernel is defined as 
\begin{align*}
    \kgexp(x,z) = \kgexp(r) = \exp\lp  - \lp\frac{r}{\theta_l}\rp^a \rp. 
\end{align*}
For this kernel, we can proceed as follows: 
\begin{align*}
    \kgexp(0) - \kgexp(r) &= 1 -  \lp e^{-r/\theta_l}\rp^a \leq 1 - \lp 1 - r/\theta_l\rp^a \leq 1 - \lp 1 - \max\{a,1\} \frac{r}{\theta_l} \rp. 
\end{align*}
The last inequality in the above display follows by considering the two cases: $0<a\leq 1$ and $1<a\leq 2$. Finally, this implies that we have 
\begin{align*}
\sqrt{K(0)-K(r)} \leq \sqrt{ \frac{\max\{a,1\} r }{\theta_l}},
\end{align*}
which implies the result with $\alpha = 1/2$ and $C_{\kgexp} = \sqrt{ \frac{\max\{a,1\}}{\theta_l}}$.

\tbf{PP Kernels.} Finally, we consider the piecewise-polynomial kernels $\kpp$ which are defined for $q=0$ and $1$ in the display below.
For the expression of $\kpp$ for other values of $q$, see \citep[\S~4.2]{williams2006gaussian} and \citep[Ch.~9]{wendland2004scattered}. 
\begin{align*}
    K_{\text{pp}, 0}(r) &= (1-r)_+^j; \qquad \qquad j = \left \lfloor \frac{D}{2} \right \rfloor + q + 1, \\ 
    K_{\text{pp}, 1}(r) &= (1-r)_+^{j+1}\big( (j+1)r + 1 \big).
\end{align*}
Since $j \geq 1$ for all choices of $D$ and $q$, we can show that 
\begin{align*}
    \sqrt{\kpp(0) - \kpp(r)} \leq \sqrt{ (j+q)r}, 
\end{align*}
which implies the required result with $\alpha=1/2$ and $C_{\kpp} = \sqrt{j+q}$. This completes the proof of \propref{prop:embed2}. \qed
\section{Proof of Theorem~\ref{theorem:regret1}}
\label{appendix:regret}
\subsection{Some concentration results}
\label{appendix:concentration1}
We collect the required concentration events required for the proof of Theorem~\ref{theorem:regret1} here. 
The first event that we require is the deviation bounds between the posterior mean $\mu_t$ and the function value, in terms of the posterior standard deviation. \citep{srinivas2012information} proved one version of this result, and it was improved upon by \citep{chowdhury2017kernelized}. Here we restate the version presented in \citep[Thm.~2]{chowdhury2017kernelized}

\begin{lemma}
\label{lemma:concentration1}
For a given $\delta>0$, under Assumption~\ref{assump:main}, the event $\mc{E}_1 = \cap_{t \geq 1} \{ |\mu_t(x) - f(x)| \leq \beta_t \sigma_t(x), \; \forall x \in \domain\}$ occurs with probability at least $1-\delta/3$, where the term $\beta_t$ is defined as $\beta_t = B + \sigma\sqrt{ 2(\gamma_t + 1 \log(3/\delta))}$. 
\end{lemma}

Next, we require a concentration result for the empirical mean $\hat{\mu}_t$ around the average function value in a cell $E$. 

\begin{lemma}
\label{lemma:concentration2}
Given a $\delta>0$, define the term $\delta_t = \frac{2\delta}{n^D \pi^2 t^2}$. Then the event $\mc{E}_2$ defined below occurs with probability at least $1-\delta/3$. 
\begin{align*}
    \mc{E}_2 = \cap_{t \geq 1} \cap_{E \in \mc{P}_t} \{ |\hat{\mu}_t(E) - \tilde{f}_E| \leq b_t(E)\}, \quad b_t(E) \coloneqq \sigma \sqrt{ \frac{ 2 \log(1/\delta_t)}{n_{t,E}}}, 
\end{align*}
where $\tilde{f}_E \coloneqq \int_E f(x)dx$ denotes the average $f$ value in the cell $E$. 
\end{lemma}
\begin{proof}
For a fixed $t$ and a fixed $E \in \mc{P}_t$, we have the following:
\begin{align*}
    P\lp |\hat{\mu}_t(E) - \tilde{f}_E| > b_t(E) \rp &= \sum_{m=1}^n P \lp n_{t,E}=m, \; |\hat{\mu}_t(E) - \tilde{f}_E| > b_t(E)\rp  \\
    & \stackrel{(a)}{\leq} \sum_{m=1}^n P\lp n_{t,E}=m\rp 2 \exp\lp - \frac{m b_t(E)^2}{2 \sigma^2} \rp = 
     \sum_{m=1}^n P\lp n_{t,E}=m \rp \delta_t = \delta_t. 
\end{align*}
The inequality \tbf{(a)} follows from an application of Chernoff's bound and the $\sigma-$subgaussian assumption on observation noise. 

Next, we observe that due to the condition in Line~\ref{algline:cond3},  Algorithm~\ref{algo:algo1} only expands cells of radius greater than $1/n$. Thus at any time $t$, we must have $|\mc{P}_t| \leq n^D$. 
This   observation allows us to apply the union bound twice as follows: 
\begin{align*}
    P\lp \mc{E}_2^c\rp & \leq \sum_{t \geq 1} \sum_{E \in \mc{P}_t}P\lp |\hat{\mu}_t(E) - \tilde{f}_E| > b_t(E) \rp \leq \sum_{t \geq 1} \sum_{E \in \mc{P}_t} \delta_t 
     \leq \sum_{t \geq 1} n^D \delta_t  = \sum_{t \geq 1} n^D \frac{2 \delta }{n^D \pi^2 t^2}
    = \frac{2\delta}{\pi^2} \sum_{t \geq 1} \frac{1}{t^2} = \delta/3. 
\end{align*}
This completes the proof.  
\end{proof}

Next, we present a lemma which formalizes the intuitive statement  that points drawn uniformly at random approximate a uniform grid. 
First we need some definitions. For any $\epsilon>0$, let $\domain_\epsilon$ represent an $\epsilon-$ cover of $\domain$ in terms of the Euclidean metric. We know that there exists a constant $0<C_4<\infty$ such that $|\domain_\epsilon| \leq C_4\epsilon^{-D}$. 
\begin{lemma}
\label{lemma:concentration3}
Introduce the event $\mc{E}_3 = \cap_{t\geq 1} \cap_{E \in \mc{P}_t} \mc{E}_3^{(t,E)}$, where we define $\mc{E}_3^{(t,E)}$ as the following event: if $n_{t,E} \geq \tilde{N}_k$ then $\DXE$ is a $2\sqrt{D}\epsilon_{t,E}$ cover of $\domain$, where 
\begin{align*}
    \epsilon_{t,E} = \lp \frac{  \log\lp C_4 n_{t,E}^D/\delta_t\rp}{n_{t,E}} \rp^{1/D}, \quad \delta_t = \frac{2\delta}{n^D \pi^2 t^2}.
\end{align*}
Then we have $P(\mc{E}_3) \geq 1-\delta/3$ for a given $\delta\in (0,1)$. 
\end{lemma}
\begin{proof}
As earlier, it suffices to show that $P(\mc{E}_3^{(t,E)}) \geq 1-\delta_t$. Now, if for some $t$ and $E$, the term $n_{t,E} <\tilde{N}_k$, then the event $\mc{E}_3^{(t,E)}$ holds by definition. So we focus on the case where $n_{t,E} > \tilde{N}_k$. 

Suppose $\epsilon_{t,E}>0$. For some $E \in \mc{P}_t$, let $z \in E_{\epsilon_{t,E}}$. Suppose the $n_{t,E}$ uniform draws from the uniform distribution over $E$ are represented by $X_1, X_2, \ldots, X_{n_{t,E}}$. Then we have the following:
\begin{align*}
    P\lp \lp \mc{E}_3^{(t,E)} \rp^c \rp &= P\lp \cup_{z \in E_{\epsilon_{t,E}}} \cap_{i=1}^{n_{t,E}} \{\|X_i - z\|_{\infty} > \epsilon_{t,E}\} \rp \leq \lp C_4 \epsilon_{t,E}^{D} \lp 1 - \epsilon_{t,E} \rp^{n_{t,E}}  \rp \\
    & \leq \exp \lp - \epsilon_{t,E}^D n_{t,E}  + \log(C_4)  + D \log(1/\epsilon_{t,E}) \rp = \delta_t. 
\end{align*}
This implies that with probability at least $1-\delta$, for every $t, E$ such that $n_{t,E} > \tilde{N}_k$, the set $\DXE$ is a $\sqrt{D}2\epsilon_{t,E}$ cover for the cell $E$.
\end{proof}

\subsection{Proof of the regret bound}
\label{proof:igain_bound}
Throughout this section, we will assume that the event $\mc{E}$ defined as  $\mc{E} \coloneqq \mc{E}_1 \cap \mc{E}_2 \cap \mc{E}_3$ holds. Note that by the concentration results of Appendix~\ref{appendix:concentration1}, we know that $\mbb{P}(\mc{E}) \geq 1 -  \delta$.

\subsubsection{Derivation of \texorpdfstring{$\igain$}{gn}  dependent bounds of Theorem~\ref{theorem:regret1}}
\label{proof:bound_cumulative}

To obtain the bound on the cumulative regret, we first derive bounds on the instantaneous regret, i.e., $\ireg(x_t) \coloneqq f(x^*) - f(x_t)$, of the points evaluated by the algorithm. 

\begin{lemma}
\label{lemma:creg1}
Suppose the algorithm evaluates a point $x_t$ at time $t$. Then, under event $\mc{E}$,  we can bound the instantaneous regret $\ireg(x_t) \coloneqq f(x^*) - f(x_t)$ as follows:
\begin{align}
   \ireg(x_t) & \leq 2 \beta_t  \sigma_t(x_t) + L(\sqrt{D}r_{E_t})^{\alpha_1} \label{eq:ireg_bound1} \quad \text{and}\\
   \ireg(x_t) & \leq 4 b_t(E_t). \label{eq:ireg_bound2}
\end{align}
\end{lemma}

\begin{proof}
Suppose at a time $t$ at which the Algorithm performed function evaluation, the algorithm selects a cell $E_t$ and let $x_t$ denote the representative point of that cell. Furthermore, let $E^*_t$ denote the cell in the partition $\mc{P}_t$ which contains the optimizer $x^*$. Then we have the following:
\begin{align*}
    f(x^*) \stackrel{(a)}{\leq} U_{t,E^*_t} \stackrel{(b)}{\leq} U_{t,E_t} \stackrel{(c)}{\leq} u_{t,E_t}^{(1)} = \mu_t(x_t) + \beta_t \sigma_t(x_t) + L(r_{E_t}\sqrt{D})^{\alpha_1}  \stackrel{(d)}{\leq}
    f(x_t) + 2\beta_t \sigma_t(x_t) + L(r_{E_t}\sqrt{D})^{\alpha_1}.
\end{align*}
In the above display, \tbf{(a)} follows from the fact that $U_{t,E_t^*}$ is a high probability upper bound on the maximum function value in the cell $E_t^*$,  \tbf{(b)} follows from the  candidate point selection rule,  \tbf{(c)} is a result of the $\min$ in the definition of $U_{t,E}$ and \tbf{(d)} follows from the definition of the event $\mc{E}_2$ in Lemma~\ref{lemma:concentration1} in which we have $|f(x_t)-\mu_t(x_t)|\leq \beta_t\sigma_t(x_t)$. 
Using this we get the following bound on the instantaneous regret: 
\begin{align}
    \label{eq:ireg1}
\ireg(x_t) = f(x^*) - f(x_t) \leq 2 \beta_t \sigma_t(x_t) + L(r_{E_t}\sqrt{D})^{\alpha_1}
\end{align}

Similarly, we can also get the following sequence of inequalities:
\begin{align*}
   f(x^*) \leq U_{t,E_t^*} \leq u_{t,E_t} \leq u_{t,E_t}^{(2)} \leq \hat{\mu}_t(E_t) + b_t(E_t) + L(r_{E_t}\sqrt{D})^{\alpha_1}
\end{align*}
Now, under the event $\mc{E}_2$, we know that $|\hat{\mu}_t(E_t) - \tilde{f}_E| \leq b_t(E_t)$, and also that $|\tilde{f}_{E_t} - f(x_t)| \leq \sup_{z_1, z_1 \in E_t}\,f(z_1) - f(z_2) \leq L (\sqrt{D}r_E)^{\alpha_1}$.
 Furthermore, if the point $x_t$ was evaluated by the algorithm at time $t$, it means that the condition on Line~\ref{algline:cond2} was not satisfied, and hence we must have $L(\sqrt{D}r_E)^{\alpha_1} \leq b_t(E_t)$.  These three facts, together with the previous display imply:
\begin{align}
   f(x^*) & \leq \tilde{f}_{E_t} + 2b_t(E_t) + L(\sqrt{D}r_E)^{\alpha_1} \leq f(x_t) + 2b_t(E_t) + 2 L(\sqrt{D}r_E)^{\alpha_1} \nonumber \\ 
   \Rightarrow \ireg(x_t) & \leq 2b_t(E_t) + 2L(\sqrt{D}r_E)^{\alpha_1} \leq 4b_t(E_t). \label{eq:ireg2}
\end{align}
This completes the proof of Lemma~\ref{lemma:creg1}. 
\end{proof}

Before proceeding, we recall   that $\Tau$ denotes the set of times $t$ at which the algorithm performed function evaluation. Furthermore, define $\Tau_1 \coloneqq \{t \in \Tau\;:\; r_{E_t} \geq \rho_0\}$ and $\Tau_2 \coloneqq \Tau \setminus \Tau_1$. Clearly by definition, we have $|\Tau|=n$ and let $|\Tau_1| = n_1$ for some $n_1 \leq n$. 

Now, under the $1-\delta$ probability event $\mc{E} = \cap_{i=1}^3\mc{E}_i$,  we can obtain the $\igain$ dependent regret bound as follows:
\begin{align*}
    \mc{R}_n &= \sum_{t\in \Tau} \ireg(x_t) = \sum_{t \in \Tau_1} \ireg(x_t) + \sum_{t\in \Tau_2} \ireg(x_t)  
     \stackrel{(a)}{\leq} \sum_{t\in \Tau_1} 3\beta_t \sigma_t(x_t) + \sum_{t \in \Tau_2}\ireg(x_t)  \\
    & \stackrel{(b)}{\leq} \sum_{t \in \Tau_1} 3\beta_t \sigma_t(x_t) + \sum_{t \in \Tau_2} 2\beta_t\sigma_t(x_t) + L(\sqrt{D}r_{E_t})^{\alpha_1} 
     \leq \sum_{t \in \Tau} 3 \beta_t \sigma_t(x_t) + \sum_{t \in \Tau_2} L(\sqrt{D}r_{E_t})^{\alpha_1} \\
    &\stackrel{(c)}{\leq} 3\beta_n \sqrt{n \igain}  + \sum_{t \in \Tau_2} L(\sqrt{D}r_{E_t})^{\alpha_1}  
    \stackrel{(d)}{\leq} 3\beta_n \sqrt{n \igain} + (n-n_1) \igain/\sqrt{n} = \tilde{\mc{O}}\lp \igain \sqrt{n} \rp.
\end{align*}
In the above display, \\
\tbf{(a)} uses the fact that for $t \in \Tau_1$, since $r_{E_t} \geq \rho_0$, as the condition on Line~\ref{algline:cond1} of Algorithm~\ref{algo:algo1} was not satisfied, we must have $\beta_t \sigma_t(x_t)\geq  L(\sqrt{D}r_{E_t})^{\alpha_1}$, \\
\tbf{(b)} uses Eq.~\ref{eq:ireg1} for the sum over $\Tau_2$, \\
\tbf{(c)} follows the approach of \citep{srinivas2012information} to bound $\sum_{t\in \Tau_2} \beta_n \sigma_t(x_t) \leq \beta_n \sum_{t} \sigma_t(x_t) \leq \beta_n \sqrt{n \igain}$, and 
\tbf{(d)} uses the fact that by definition of the input parameter $\rho_0$, for all $t \in \Tau_2$, we must have $ L(\sqrt{D}r_{E_t})^{\alpha_1} \leq  \igain/\sqrt{n}$
Since $n- n_1 \leq n$ the final statement follows. 


\vspace{1em}
Next, we obtain the bound on the bound on simple regret in terms of the information gain. 

\begin{lemma}
\label{lemma:simple_regret1}
Under the event $\mc{E}$, for the point $\recc$ returned by \algoref{algo:algo1}, we have $\sreg(z_n) = f(x*) - f(z_n) = \tOh{\igain/\sqrt{n}}$. 
\end{lemma}
\begin{proof}
The proof follows from the bound on cumulative regret derived in the previous section, and the \recommend~function used for returning the point $z_n$. More specifically, we consider the two possible cases:
\begin{itemize}
    \item Suppose that the \recommend~function returns a point $x_\tau$ where $\tau = \argmin_{t \in \Tau}\beta_t \sigma_t(x_t)$. In this case we have 
    \begin{align*}
        f(x^*) - f(\recc) & \leq \beta_{\tau}\sigma_\tau(x_\tau) \stackrel{(a)}{\leq} \frac{1}{n} \sum_{t \in \Tau}\beta_t\sigma_t(x_t) \stackrel{(b)}{=} \tOh{ \frac{\igain}{\sqrt{n}}}. 
    \end{align*}
    In the above display, \tbf{(a)} follows from the definition of $\tau$ and that $\min$ is smaller than average, while \tbf{(b)} uses the corresponding bound on the cumulative regret.

    \item Suppose the \recommend~function returns a point $\recc = x_{E_n}$ where $E_n = \argmin_{E \in \mc{P}_{t_n}} r_E$ and $t_n$ is the time at which \algoref{algo:algo1} stops, i.e., performs the $n^{th}$ function evaluation. Again, in this case, we have the following:
    \begin{align*}
        f(x^*) - f(z_n) \leq L(\sqrt{D}r_{E_n})^{\alpha_1} \leq \min_{t \in \Tau} \beta_t \sigma_t(x_t). 
    \end{align*}
    The rest of the proof proceeds according to the steps used in the previous case, and we again get that $\sreg(\recc) = \tOh{\igain/\sqrt{n}}$. 
\end{itemize}
\end{proof}
This completes the proof of the first part of Theorem~\ref{theorem:regret1}.

\subsubsection{Derivation of smoothness dependent bounds of Theorem~\ref{theorem:regret1}}
\label{proof:bound_simple}
In this section, we derive the bounds on the simple and cumulative  regret in terms of the smoothness parameters of the kernels. 
We begin with a loose upper bound on $t_n$, the number of rounds after which \algoref{algo:algo1} stops. 
\begin{lemma}
\label{lemma:tn_bound}
We have $t_n \leq T_n \coloneqq ((n/2)^D + n)$. 
\end{lemma}
\begin{proof}
The algorithm expands cells only with radius greater than or equal to $1/n$ (Line~\ref{algline:cond3} of \algoref{algo:algo1}). Furthermore, every time a cell is expanded, at least $2^D$ new cells are added to the partition $\mc{P}_t$. Thus the number of rounds in which cell expansion took place, can be upper bounded by $\frac{(1/n)^{-D}}{(1/2)^{-D}} = (n/2)^D$. Furthermore, since the algorithm stops after making $n$ function evaluations, we get the required bound $t_n \leq (n/2)^D + n$. 
\end{proof}

Next, we present a bound on the number of times the algorithm must evaluate a cell before expanding it. 
\begin{lemma}
\label{lemma:smooth2}
Let $m_E$ denote the total number of times a cell $E$ is evaluated by the algorithm before being expanded. Then we have $m_E \leq 2\log(1/\delta_{T_n})/\big( LD r_E^{2\alpha_1} \big)$ if $r_E \geq \rho_0$ and $m_E \leq 2\log(1/\delta_{T_n})/\big( LD r_E^{2(k+\alpha)} \big)$ if $r_E < \rho_0$. Recall that $T_n$ was introduced in Lemma~\ref{lemma:tn_bound} and $\delta_t$ was defined in Lemma~\ref{lemma:concentration2}. 
\end{lemma}
\begin{proof}
The result follows  from the cell expansion conditions given in Lines~\ref{algline:cond1}~and~\ref{algline:cond1} (corresponding to $r_E \geq \rho_0$) and Line~\ref{algline:cond3} (corresponding to $r_E<\rho_0$) in \algoref{algo:algo1},  i.e., $b_t(E) \leq L(\sqrt{D}r_E)^{\alpha_1}$ and $b_t(E) \leq L(\sqrt{D}r_E)^{(k+\alpha)}$ respectively.

Suppose at some time $t$, we have $E=E_t$ and $r_E \geq \rho_0$. Then if the algorithm evaluates the function $f$ in round $t$, we must have $b_t(E) \geq L(\sqrt{D}r_E)^{\alpha_1}$.  Now, since $b_t = \sqrt{2 \log(1/\delta_t)/n_{t,E}}$, this implies $n_{t,E} \leq \frac{2\log(1/\delta_t)}{L^2 D r_E^{2\alpha_1}} \leq \frac{2\log(1/\delta_{T_n})}{L^2 D r_E^{2\alpha_1}}$.  This gives us an upper bound on the number of times the algorithm evaluates points in a  cell $E$ before expanding. The result for the case $r_E \leq \rho_0$ follows in a similar manner. 
\end{proof}

\begin{lemma}
\label{lemma:smooth3}
Suppose \algoref{algo:algo1} selects a cell $E \neq \domain$ at any time $t$ according to Line~\ref{algoline:candidate} of \algoref{algo:algo1}, and further assume that the cell $E$ was added to the partition by expanding a cell $F \supset E$ at some time $s<t$. 
Then  $\sup_{x \in E} f(x^*) - f(x) = \tOh{r_F^{\alpha_1}}$ if $r_F \geq \rho_0$ and $\sup_{x \in E} f(x^*) - f(x) = \tOh{r_F^{(k+\alpha)}}$ if $r_F \leq \rho_0$. Recall that $\alpha_1 \coloneqq \max \{\alpha, \min\{1, k\}\}$. 
\end{lemma}
\begin{proof}

We consider three possible ways in which the cell $F$ was expanded at some time $s<t$ to introduce $E$ to the partition $\mc{P}_s$. The first two cases correspond to the condition $r_F \geq \rho_0$ while the third cases corresponds to the condition $r_F \leq \rho_0$. 

\emph{Case~1:} $E$ was added by expanding its parent cell $F$ by a call to \expand~algorithm  from Line~\ref{algline:cond1} at some time $s$. 
Then we have by construction, $u_E^{(0)} = \mu_s(x_s) + \beta_s \sigma_s(x_s) + L(\sqrt{D}r_F)^{\alpha_1}$. Now, for any $x \in E$, we have $|f(x) - f(x_s)| \leq L(\sqrt{D}r_F)^{\alpha_1}$ since $x \in E \subset F$. Then  by~\eqref{eq:ireg1}, we have 
\begin{align*}
    \ireg(x_s)  &= f(x^*) - f(x_s) \leq 2\beta_s \sigma_s(x_s) + L(\sqrt{D}r_F)^{\alpha_1}  \\
    \Rightarrow f(x^*) - f(x) & \leq 2 \beta_s \sigma_s(x_s) +2 L(\sqrt{D}r_F)^{\alpha_1} 
     \leq 4L(\sqrt{D}r_F)^{\alpha_1} 
\end{align*}

\emph{Case~2:} $E$ was added by expanding its parent cell $F$  by a call to \expand~algorithm  from Line~\ref{algline:cond2} at  time $s<t$. 

In this case, we have the following sequence (here $x_s$ is the candidate point in $F$ which was selected by the algorithm at time $s$, and $x$ is any point in $E$). 
\begin{align*}
    f(x^*) - f(x) & \leq f(x^*) - f(x_s) + f(x_s) - f(x) \leq 2b_s(F) + 2L(\sqrt{D}r_F)^{\alpha_1}   + f(x_s) - f(x) \\
    f(x^*) - f(x)  &\leq 4L(\sqrt{D}r_F)^{\alpha_1}   + f(x_s) - f(x) 
     \leq 5L(\sqrt{D}r_F)^{\alpha_1}.
\end{align*}

\emph{Case~3:} $E$ was added by calling \expand~from Line~\ref{algline:cond2} from some ancestor cell $F \supset E$ at  time $s$. 

Let $x$ denote any point in the cell $E$. Since the cell $E$ was selected by the algorithm at time $t$, we must have the following: 
\begin{align*}
        f(x^*) &\leq u_{t,E}^{(0)}  \stackrel{(a)}{=} (\hat{f}_E + \ttt{err}) + \ttt{err} \stackrel{(b)}{\leq} f(x_E) + \ttt{err} \\
    & \stackrel{(c)}{\leq} f(x) + L (\sqrt{D}r_E)^{\alpha_1} + \ttt{err} \stackrel{(d)}{=}  f(x) +  2\ttt{err},
\end{align*}
where we use \ttt{err} to denote the output of the \maxerr~subroutine called by the \expand~algorithm from the cell $F$ which was refined to introduce $E$ into the partition $\mc{P}_s$ at some time $s < t$. The equality \tbf{(a)} in the above display follows from the assignment of $u_{t,E}^{(0)}$ in Line~\ref{algoline:expand1} of \algoref{algo:expand}, while \tbf{(b)} follows from the error bound of the local-polynomial estimator, \tbf{(c)} uses the fact that $|f(x)-f(x_E)| \leq L(\sqrt{D}r_E)^{\alpha_1}$ and \tbf{(d)} uses the relation between $r_E$ and \ttt{err} according to Line~\ref{algoline:expand2} of \algoref{algo:expand}. 

To complete the proof it remains to show that this term \ttt{err} is also $\mc{O} \lp r_F^{k+\alpha}\rp$. 
First we note that  due to the event $\mc{E}_3$, we know from Lemma~\ref{lemma:concentration3} that the elements of $\mc{D}_{\domain}^{(F)}$ form a $2\sqrt{D}\epsilon_{s,F}$ cover for $F$.
For an $x \in F$ and $h>0$, let $e_S$ and $e_D$ denote the two terms in the error bound for LP estimators introduced in Lemma~\ref{lemma:LP1}. 
Combining the  existing analysis of LP estimators over uniform grids in \citep{nemirovski2000topics} with Lemma~\ref{lemma:concentration3} presented earlier, we know that  there exist constants $C_5, C_6>0$ such that 
\begin{align*}
    e_D &\leq \max_{z \in F}\;(1 + \|\vec{w}_z\|_1) \Phi_D^k \leq C_5 L (r_F)^{k+\alpha}, \quad \text{and} \\
    e_S &\leq  \max_{z \in F} \|\vec{w}_z\|_2 \sigma \sqrt{\log(n_{s,F}/\delta_{T_n})} \leq 
    C_6 \sigma  \sqrt{ \frac{ \log(n_{s,F}/\delta_t)}{n_{s,F}}}. 
\end{align*}
The first inequality uses \holder~assumption on $f$ along with \citep[Eq.~1.45]{nemirovski2000topics}, while the second result follows from an application of Lemma~1.3.1 of \citep{nemirovski2000topics} along with  Lemma~\ref{lemma:concentration3}. 

From the condition used for expanding the cell we have that $ \sigma \sqrt{ \frac{ \log(n_{s,F}/\delta_t)}{n_{s,F}}} \leq L (r_F)^{k+\alpha}$ which implies that $\ttt{err} \leq (C_5 + C_6) L (r_F)^{k+\alpha} \coloneqq C_7 (r_F)^{k + \alpha}$. This completes the proof.

\end{proof}
\begin{remark}
\label{remark:ireg_bound}
Note that in the first two cases in the proof of Lemma~\ref{lemma:smooth3} above, we have $r_F = 2r_E$. Furthermore, in the third case, due to the choice of $\tilde{r}$   in the \expand~algorithm, we have that  $ (1/2)(r_F)^{(k+\alpha)/\alpha_1} \leq r_E \leq (r_F)^{(k+\alpha)/\alpha_1}$. 
As a result, in all these cases, if a cell $E$ is selected by the algorithm at time $t$, then we must also have $\sup_{x \in E}\; f(x^*) - f(x) \leq \tilde{\mc{O}}\lp r_E^{\alpha_1} \rp$. 
\end{remark}

Next, we obtain upper bounds on the radius of the smallest cell activated by the algorithm. Before stating the result, we recall that $\Tau \subset \{1, 2, \ldots, t_n\}$ denotes the times at which the algorithm evaluates $f$. $\Tau$ is further partitioned into  $\Tau_1 = \{t \in \Tau: r_{E_t}\geq \rho_0\}$ and $\Tau_2 = \Tau \setminus \Tau_1$. Furthermore, we have $|\Tau|=n$, $|\Tau_1|=n_1$ and $|\Tau_2|=n_2 = n-n_1$.

\begin{lemma}
\label{lemma:smooth5}
Let $2^{-j^*}$ be the radius of the smallest active cell in the partition $\mc{P}_{t_n}$ where $t_n$ denotes the time at which the algorithm stops. Then we have $2^{-j^*} = \tOh{n^{-1/(D+2\alpha_1)}}$   if $n_2 \leq n/2$ and 
$2^{-j^*} = \tOh{n^{-1/(D+2\nu)}}$ if $n_2 \geq n/2$. 
\end{lemma}
\begin{proof}

Suppose $2^{-j^*}$ is the radius of the smallest active cell when the algorithm halts. Clearly, this radius must be smaller than the radius of the smallest cell that has been expanded by the algorithm, denoted by $2^{-j_1^*}$.  We introduce the notation $j_0 \coloneqq \lceil \log(\sqrt{n}/\igain)/\log(2) \rceil$, and also recall from Lemma~\ref{lemma:smooth2} that for $E$ with $r_E \geq \rho_0$, the number of times the algorithm evaluates the cell $E$ is upper bounded by $m_E \coloneqq C_8 r_E^{-2\alpha_1}$ where $C_8 = 2\log(1/\delta_{T_n})/(L^2D)$. 

First we consider the case when $n_2 \leq n/2$. 
 To bound $2^{-j_1^*}$ we use the following two observations:(a) each cell of radius $2^{-j}, j \leq j_0$ is evaluated no more than $C_82^{2j\alpha_1}$ times, and (b) there are $2^{jD}$ cells of radius $2^{-j}$. Thus, in order for the algorithm to expand a cell of radius $2^{-j_2}$ for some $j_2 \geq 0$ the number of evaluations can be upper bounded by 
\begin{align*}
    \mc{B}_1(j_2) = \sum_{j=0}^{j_2-1} C_82^{j(2\alpha_1 + D)} = C_8\lp 2^{j_2(2\alpha_1 + D)} - 1 \rp. 
\end{align*}
Let $j_2^*$ denote the largest value of $j_2$ for which  $\mc{B}_1(j_2) \leq n_1$. Then we have \begin{align*}
    \mc{B}_1(j_2^*+1) &= C_8\lp 2^{(j_2^*+1)(2\alpha_1+D)} - 1 \rp  > n_1 \geq n/2 \\
    \Rightarrow 2^{j_2^*(2\alpha_1+D)}& \geq \lp \frac{n}{C_8 2^{(2\alpha_1+D+1)}} \rp 
    \Rightarrow 2^{-j_2^*}  \leq \lp \frac{n}{C_8 2^{2\alpha_1+D+1}}\rp^{1/(2\alpha_1+D)} = \Oh{\lp \frac{n}{\log n} \rp^{\frac{1}{2\alpha_1+D}}}. 
\end{align*}
Since $2^{-j_2^*}$ is an upper bound on $2^{-j_1^*}$, this proves the required result. 

Next, we consider the cases in which we have $n_2 \geq n/2$ and $\igain = \Omega(\sqrt{n})$. 
 Introduce the notation $\zeta = (k+\alpha)/\alpha_1$. Note that due to the assumption on $\igain$, we must have that $j_0= \Oh{1}$ which implies that $n_1$ is also $\Oh{1}$.   Similar to the previous case, we will obtain a further upper bound on the radius $2^{-j_1^*}$ using the following two observations: (a) each cell of radius $2^{-j}$ for $j \geq j_0$ is evaluated no more than $C_8 2^{2j(k+\alpha)}$ times, and (b) there are $2^{jD}$ cells of radius $2^{-j}$. Using these two observations, and for any integer $l_0>0$, the total number of observations required before the algorithm expands all cells of radius $2^{-j_0 \zeta^{l_0}}$ can be upper bounded by 
\[
\mc{B}_2(l_0) \coloneqq \sum_{l=0}^{l_0} C_8 2^{j_0 \zeta^l(2(k+\alpha) + D)} \leq C_8(l_0+1) 2^{j_0\zeta^{l_0}(2(k+\alpha) + D)} \coloneqq C_9 2^{j_0 \zeta^{l_0} (2(k+\alpha)+D)}, 
\]
where the inequality holds due to the fact that the last term dominates the summation.
To find an upper bound on $2^{-j_1^*}$, we find the largest value of $l_0$ for which the quantity $C_92^{j_0\zeta^{l_0}(2(k+\alpha)+D)}$ is smaller than $n_2$. For this we assume that $n$ (and hence $n_2$) is large enough that there exists a constant $0<c<1$ such that $c n_2 \leq C_9 2^{j_0\zeta^{l_0}(2(k+\alpha)+D)} \leq n_2$ (see  Remark~\ref{remark:n_large_enough}). From this we get the important inequality 
\begin{equation}
 2^{-j_0\zeta^{l_0}} = \mc{O} \lp \lp \frac{n_2}{\log n_2}\rp ^{-1/(2(k+\alpha)+D)} \rp = \mc{O} \lp \lp \frac{n}{\log n} \rp^{-1/(2(k+\alpha)+D)} \rp,
\end{equation}
where the second inclusion in the above display uses the assumption that $n_2 \geq n/2$. 
\end{proof}

Finally, we have all the ingredients to present the bounds on the simple and cumulative regret. 
\begin{lemma}
\label{lemma:smooth6}
Under the event $\mc{E}$, if $n_2 \leq n/2$, we have the following bounds: 
\begin{align*}
    \sreg & = \tOh{n^{-\frac{\alpha_1}{D + 2\alpha_1}}}, \quad \creg = \tOh{ n^{ \frac{ D + \alpha_1}{D + 2\alpha_1}}}, 
\end{align*}
while if $n_2 \geq n/2$ and $\igain = \Omega(\sqrt{n})$, we have the following:
\begin{align*} 
\sreg & = \tOh{ n^{-\frac{(k+\alpha)}{D + 2(k+\alpha)}}}, \quad \creg = \tOh{ n^{ \frac{ 2(k+\alpha) - \alpha_1 + D}{2(k+\alpha) + D}}}. 
\end{align*}
\end{lemma}

\begin{proof}
The bounds on the simple regret follow by a combination of the results derived in Lemma~\ref{lemma:smooth3} and Lemma~\ref{lemma:smooth5}. More specifically, if $n_1 \geq n/2$, we cannot guarantee that the smallest active cell in $\mc{P}_{t_n}$ has a radius smaller than $\rho_0$. Thus,  we get the following by using the first part of  Lemma~\ref{lemma:smooth2}:
\begin{equation}
    \sreg(z_n) \leq \tOh{2^{-j_1^*\alpha}} \leq \tOh{n^{-\alpha/(D + 2\alpha)}}. 
\end{equation}
Next, for the second case, we know that the smallest active cell must have radius smaller than $\rho_0$. Thus by another application of Lemma~\ref{lemma:smooth2} and by using the bound on $2^{-j^*}$ derived in Lemma~\ref{lemma:smooth6}, we have 
\begin{align*}
    \sreg(z_n) \leq \tOh{ 2^{-j_1^* (k+\alpha) }} \leq \tOh{2^{-j_0\zeta^{l_0} (k+\alpha)}} \leq \tOh{n^{-(k+\alpha)/(2(k+\alpha)+D)}}. 
\end{align*}
This completes the proof of the bounds on simple regret.

Next, we proceed towards the bounds on cumulative regret. 
For the first case where $n_1 \geq n/2$,  we have the following (the term $j_1^*$ was introduced in the proof of Lemma~\ref{lemma:smooth5}): 
\begin{align*}
    \creg &= \sum_{t \in \Tau} f(x^*) - f(x_t) = \sum_{t\in \Tau_1} f(x^*) - f(x_t) + \sum_{t \in \Tau_2} f(x^*) - f(x_t) \\ 
    &  \stackrel{(a)}{\leq} \tOh{\sum_{j=1}^{j_1^*} 2^{2j\alpha_1} 2^{jD} 2^{-j\alpha_1} } + n \tOh{2^{-j_1^*\alpha_1}} \\
    & = \tOh{ 2^{j_1^*(D+\alpha_1)}} + \tOh{ n 2^{-j_1^*\alpha_1}} \stackrel{(b)}= \tOh{ n^{ \frac{D+\alpha_1}{D+2\alpha_1}}}. 
\end{align*}
In the above display, \tbf{(a)} uses the fact that the number of times the algorithm evaluates a cell of radius $2^{-j}$ is upper bounded by $m_E = \tOh{2^{2j\alpha_1}}$ for $j \leq j_0$, and \tbf{(b)} uses the fact that $2^{-j_1^*} = \tOh{n^{-1/(D+2\alpha_1)}}$ as derived in Lemma~\ref{lemma:smooth5}.

Next,  we consider the case where $n_2 \geq n/2$ and $\igain = \Omega(\sqrt{n})$. 
  Since we know from the proof of Lemma~\ref{lemma:smooth5} that in this case $j_0 = \Oh{1}$, this implies that we must have $\sum_{t \in \Tau_1} f(x^*) - f(x_t)  = \mc{O}(1)$ as well.  To complete the proof, we will show that the remaining term $\sum_{t\in \Tau_2} f(x^*) - f(x_t) = \tOh{n^{(D+2(k+\alpha)-\alpha_1)/(D+2(k+\alpha))}}$ which dominates the $\tOh{1}$ term. Again, we proceed as follows (the terms $l_0$, $\zeta$ and $j_0$ were introduced in the proof of Lemma~\ref{lemma:smooth5}):
\begin{align*}
    \sum_{t \in \Tau_2} f(x^*) - f(x_t) & \stackrel{(a)}{\leq} \tOh{\sum_{l=0}^{l_0} 2^{j_0\zeta^l2(k+\alpha)} 2^{j_0\zeta^l D} 2^{-j_0\zeta^l\alpha_1}} + n\tOh{2^{-j_0 \zeta^{l_0} (k+\alpha)}} \\
    & = \tOh{ 2^{j_0\zeta^{l_0}(D+2(k+\alpha)-\alpha_1) }} + \tOh{ n 2^{-j_0 \zeta^{l_0}(k+\alpha) } } \\
    & \stackrel{(b)}{=} \tOh{ n^{(D+2(k+\alpha)-\alpha_1)/(D+2(k+\alpha))}} + \tOh{ n^{1 - (k+\alpha)/(D + 2(k+\alpha))}} \\
    & = \tOh{n^{(D+2(k+\alpha)-\alpha_1)/(D+2(k+\alpha))}}. 
\end{align*}
The inequality \tbf{(a)} in the display above uses the fact that a cell of radius $2^{-j_0\zeta^l}$ is evaluated at most $\tOh{2^{2(k+\alpha) j_0\zeta^l}}$ times from Lemma~\ref{lemma:smooth2}, and that every evaluated point in such as cell is at most $\tOh{2^{-j_0\zeta^l \alpha_1}}$ sub-optimal. The relation \tbf{(b)} uses the result from Lemma~\ref{lemma:smooth5} that $2^{-j_0\zeta^{l_0}} = \tOh{n^{-1/(D + 2(k+\alpha))}}$.

\end{proof}

\begin{remark}
\label{remark:n_large_enough}
During the proof of Theorem~\ref{theorem:regret1}, we made certain assumptions on the budget $n$ being \emph{large enough}. Here we collect those assumptions. 
\begin{itemize}
\item First, as stated in Assumption~\ref{assump:embed}, we assume that $n$ is large enough to ensure that $C_2 \leq \log(n)$ where the constant $C_2$ was introduced in the proof outline of Proposition~\ref{prop:embed1}. 
    \item In Lemma~\ref{lemma:smooth5}, for value of $\nu$ such that  $\igain = \Omega(\sqrt{n})$, we have $j_0 =\Oh{1}$ which implies that the term $n_1 = |\Tau_1|$ is also $\Oh{1}$. Since we assume $n_2 \geq n/2$, this implicitly requires that $n$ is large enough such that $n/2 \geq n_1$. 

    \item In the same derivation, we assume that there exists a universal constant $c\in(0,1/2)$ such that $c n \leq C_9 2^{j_0\zeta^{l_0}(D+2\alpha)} \leq n/2$. The term $j_0$ is $\Oh{1}$ and $C_9$ is $\Oh{\log n \log\log n}$. Thus this assumption essentially requires that $n$ is of the order $C_{10}^{\zeta^{l_0}}$ for some integer $l_0 \geq 1$ and the constant $C_{10}= 2^{j_0(D+2(k+\alpha))}$. 
\end{itemize}
\end{remark}


\newpage 
\section{Pseudo-Code of Heuristic Algorithm from Section~\ref{subsec:practical}}
\label{appendix:heuristic}
We now present the steps of the \ttt{Heuristic} algorithm described in Section~\ref{subsec:practical}. This algorithm proceeds in a manner similar to the LP-GP-UCB algorithm with $k=0$, with two changes:
\begin{enumerate}
    \item The first change is in the definition of the upper confidence bound for a cell $E$ in Line~5 of Algorithm~\ref{algo:heuristic}. Here $U_{t,E}$ is the minimum of only two terms $u_{t,E}^{(i)}$ for $i=1,2$ and we do not have the term $u_{t,E}^{(0)}$ which was obtained by a call to \expand algorithm. 
    
    \item The partition $\mc{P}_t$ is updated at every step of the algorithm by calling a standard decision tree regression function with MSE criterion (in the experiments we used the implementation from \ttt{sklearn} package in python). 
\end{enumerate}

\begin{algorithm}[H]
\SetAlgoLined
\SetKwInOut{Input}{Input}\SetKwInOut{Output}{Output}
\SetKwFunction{LocalPoly}{LocalPoly}
\SetKwFunction{Expand}{Expand}
\newcommand\mycommfont[1]{\ttfamily\textcolor{blue}{#1}}
\SetCommentSty{mycommfont}

\KwIn{$n$, $K$, $B$, $\alpha$, $L$.}
 \tbf{Initialize:}~ $t=1$, $n_e=0$, $\mc{P}_t = \{ \domain \}$,  $\mc{D}_t = \emptyset$\;
 \While{$n_e < n$}{
\For{$E \in \mc{P}_t$}
{  
Draw $x_{t,E} \sim \texttt{Unif}(E)$ \\
$U_{t,E} = \min\{u_{t,E}^{(1)},\; u_{t,E}^{(2)}\}$\\
} 
$E_t \in \argmax_{E \in \mc{P}_t}\; U_{t,E}$, $\;$  $x_t = x_{t,E_t}$\;
\BlankLine
{Observe} $y_t = f(x_t) + \eta_t$\; 
{Update} $\mu_t, \sigma_t$\; 
$\mc{D}_t \leftarrow \mc{D}_t \cup \{(x_t, y_t)\}$ \\
$n_e \leftarrow n_e + 1$ \\ 
$t \leftarrow t+1$ \\
Update $\mc{P}_t$ using Regression Tree on $\mc{D}_t$
 }
 \caption{\ttt{Heuristic} Algorithm}
 \label{algo:heuristic}
\end{algorithm}

\section{Details of Experiments}
\label{appendix:experiments}
For both experiments, we used the values of $B=1.0$ and $L=\sqrt{2}$~(for \lpgpucb and \heuristic), and used the \matern kernel with $\nu=2.5$. The length-scale of the kernel was chosen by maximizing the ML estimate with the first $5$ samples (which were drawn uniformly over $\domain$). For maximizing the acquisition function (in IGPUCB, EI and PI), we used the L-BFGS optimizer of \ttt{scipy.optimize.minimize} function with $10$ random-restarts.

\paragraph{Details of Hyperparameter Tuning Task.} In this task, we studied the performance of the algorithms in tuning the hyperparameters of a Convolutional Neural Network~(CNN) with 2 conv layers and 2 fully connected layers. The input space was five dimensional with the following hyperparameters: 
\begin{itemize} \itemsep 0em
\item \ttt{batch\_size} with values in $\{10, 11, \ldots, 500\}$. 
\item \ttt{kernel\_1} and \ttt{kernel\_2} with values in $\{1, 5, 9\}$. 
\item \ttt{hidden\_nodes} in first linear layer, with values in $\{2,3,\ldots, 40\}$
\item \ttt{log\_learning\_rate} with values in the set $[-6,0]$
\end{itemize}

All the hyperparemters were encoded with values in the unit interval $[0,1]$. For the integer valued arguments, we discretized the unit interval $[0,1]$ while for the real valued arguments we scaled and shifted the values to map it to $[0,1]$. 

The MNIST dataset was used for training and testing the CNN. To make the problem harder, we used only $10000$ randomly selected training samples out of the available $50000$, and also only ran the training loop for $5$ epochs. The entire test set of $10000$ samples was used for computing the test accuracy of the learned model. 

\end{appendix}

\end{document}